\documentclass{article}
\usepackage[numbers, compress]{natbib}

\usepackage{svaiter}
\usepackage[utf8]{inputenc} % allow utf-8 input
\usepackage{hyperref}       % hyperlinks
\usepackage[margin=1.1in]{geometry}
\usepackage{url}            % simple URL typesetting
\usepackage{amsfonts}       % blackboard math symbols
\usepackage{microtype}      % microtypography
\usepackage{relsize}
\usepackage{parskip}
\date{}
\usepackage{authblk}
\usepackage{svaiter}
\usepackage{wrapfig}
\usepackage{enumitem}

\makeatletter
\def\thm@space@setup{%
  \thm@preskip=\parskip \thm@postskip=0pt
}
\makeatother

\usepackage{xcolor}
\definecolor{linkcolor}{RGB}{83,83,182}
\definecolor{citecolor}{RGB}{128,0,128}
\hypersetup{
    colorlinks=true,
    citecolor=citecolor,
    linkcolor=linkcolor,
    urlcolor=linkcolor
}

\usepackage{color}

\newcommand{\Ww}{\mathcal{W}}
\newcommand{\Cc}{\mathcal{C}}
\newcommand{\Hh}{\mathcal{H}}
\newcommand{\Pp}{\mathcal{P}}

\newcommand{\Ff}{\mathcal{F}}

\newcommand{\Bb}{\mathcal{B}}
\newcommand{\Mm}{\mathcal{M}}
\newcommand{\Yy}{\mathcal{Y}}
\newcommand{\Nn}{\mathcal{N}}

\def\EE{\mathbb{E}}

\def\MSE{\textup{MSE}}

\title{On the Universality of Graph Neural Networks\\
on Large Random Graphs}

\author{Nicolas Keriven\thanks{CNRS \& GIPSA-lab. 11 rue des Math\'ematiques, 38400 St-Martin-d’Her\`es, France. \url{firstname.name@gipsa-lab.grenoble-inp.fr}}
\qquad
Alberto Bietti\thanks{NYU Center for Data Science, New York, USA.
\url{firstname.name@nyu.edu}}
\qquad
Samuel Vaiter\thanks{CNRS \& IMB, Universit\'e de Bourgogne. 9 avenue Alain Savary, 21000 Dijon, France. SV is partly supported by ANR JCJC GraVa (ANR-18-CE40-0005). \url{firstname.name@u-bourgogne.fr} }
}

\begin{document}

\maketitle

\begin{abstract}
  We study the approximation power of Graph Neural Networks (GNNs) on latent position random graphs.
  In the large graph limit, GNNs are known to converge to certain ``continuous'' models known as c-GNNs, which directly enables a study of their approximation power on random graph models. In the absence of input node features however, just as GNNs are limited by the Weisfeiler-Lehman isomorphism test, c-GNNs will be severely limited on simple random graph models. For instance, they will fail to distinguish the communities of a well-separated Stochastic Block Model (SBM) with constant degree function.
  Thus, we consider recently proposed architectures that augment GNNs with unique node identifiers, referred to as Structural GNNs here (SGNNs).
  We study the convergence of SGNNs to their continuous counterpart (c-SGNNs) in the large random graph limit, under new conditions on the node identifiers. We then show that c-SGNNs are strictly more powerful than c-GNNs in the continuous limit, and prove their universality on several random graph models of interest, including most SBMs and a large class of random geometric graphs. Our results cover both permutation-invariant and permutation-equivariant architectures.
\end{abstract}

\section{Introduction}

Graph Neural Networks (GNNs) are deep architectures defined over graph data that have garnered a lot of attention in recent years. They represent the state-of-the-art in many graph Machine Learning (graph ML) problems, and have been successfully applied to e.g. node clustering \cite{Bruna-clust}, semi-supervised learning \cite{Kipf-SSL}, quantum chemistry \cite{Gilmer-MP}, and so on. See \cite{Bronstein-GDL, Wu-review, Hamilton-book, OGB} for reviews.

As the universality of Multi-Layers Perceptrons (MLP) is one of the foundational theorems in deep learning -- that is, any continuous function can be approximated arbitrarily well by an MLP -- in the last few years the approximation power of GNNs has been a topic of great interest.
In the absence of special node features, i.e. when one has only access to the graph structure, the crux of the problem has been proven to be the capacity of GNNs to solve the \emph{graph isomorphism problem}, that is, deciding when two graphs are permutations of each other or not (a difficult problem for which no polynomial algorithm is known \cite{Babai-GI}). Indeed, this property is directly linked to the approximation power of GNNs~\cite{Bruna-GI, Lelarge-power}. In this light, the landmark paper \cite{Xu-WL} proves that classical GNNs are at best as powerful as the famous Weisfeiler-Lehman (WL) test \cite{WL} for graph isomorphism. Since then, several works \cite{Xu-WL, Maron-powerful, Bruna-GI} have derived new architectures, for instance involving high-order tensors \cite{Maron-powerful}, with improved discriminative power equivalent to ``higher-order'' variants of the WL test.
In another line of works, several recent papers have advocated the use of \emph{unique node identifiers} \cite{Loukas-depth, Loukas-GI}, with strategies to preserve the permutation-equivariance/invariance of GNNs coined Structural Message Passing (SMP) in \cite{Vignac-smp} (see Sec.~\ref{sec:sgnn}). We call these models Structural GNN (SGNN) here. SGNNs have been proved to be strictly more powerful than the WL test in \cite{Vignac-smp}, and even \emph{universal} on graphs with bounded degrees, however for powerful layers that cannot be implemented in practice. %To our knowledge, the exact approximation power of GWNNs is still open.

When the size of the graphs grows, the notion of graph isomorphism becomes somewhat moot: large graphs might share properties (number of well-connected communities, and so on) but are never isomorphic to each other. GNNs have nevertheless proven successful in identifying their large-scale structures, e.g. for node clustering \cite{Bruna-clust}. Several papers have therefore used tools from random graph theory and graphons to study the behavior of GNNs \emph{in the large-graph limit}. In \cite{us,Ruiz-transferability}, GNNs are shown to converge to limiting ``continuous'' architectures (coined c-GNNs in \cite{us}). A few works have studied the discriminative power of c-GNNs on graphons~\cite{Magner-power}, however, analogously to how the WL test will fail on regular graphs, c-GNNs are severely limited on graph models with almost-constant degree function (Fig.~\ref{fig:sbm}), and the question is still largely open.

\begin{wrapfigure}{r}{.48\textwidth}
    \vspace{-15pt}
    \centering
    \includegraphics[width=0.45\textwidth]{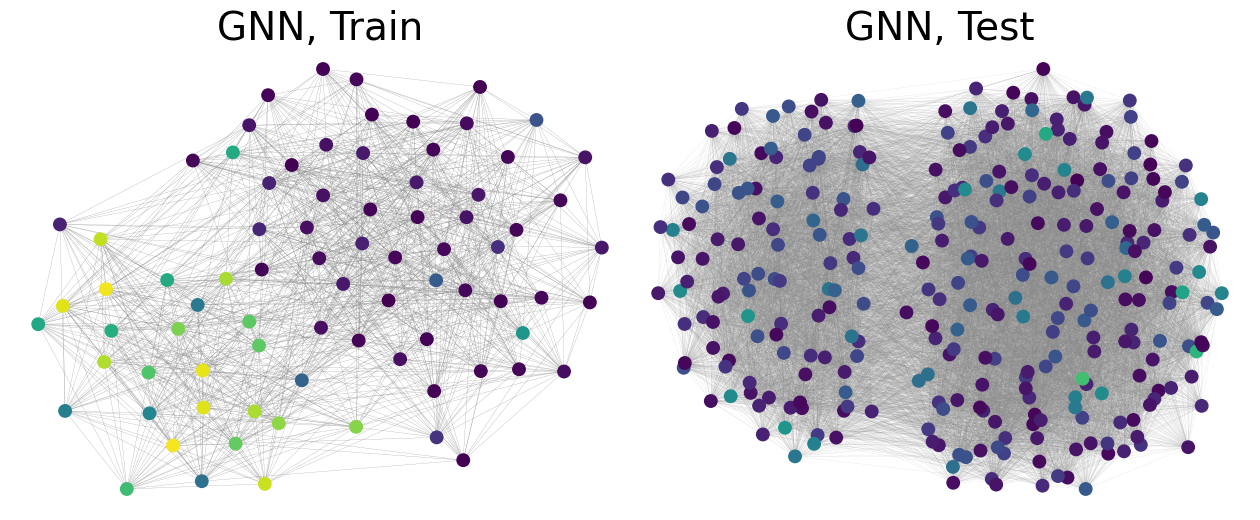}\\
    \includegraphics[width=0.45\textwidth]{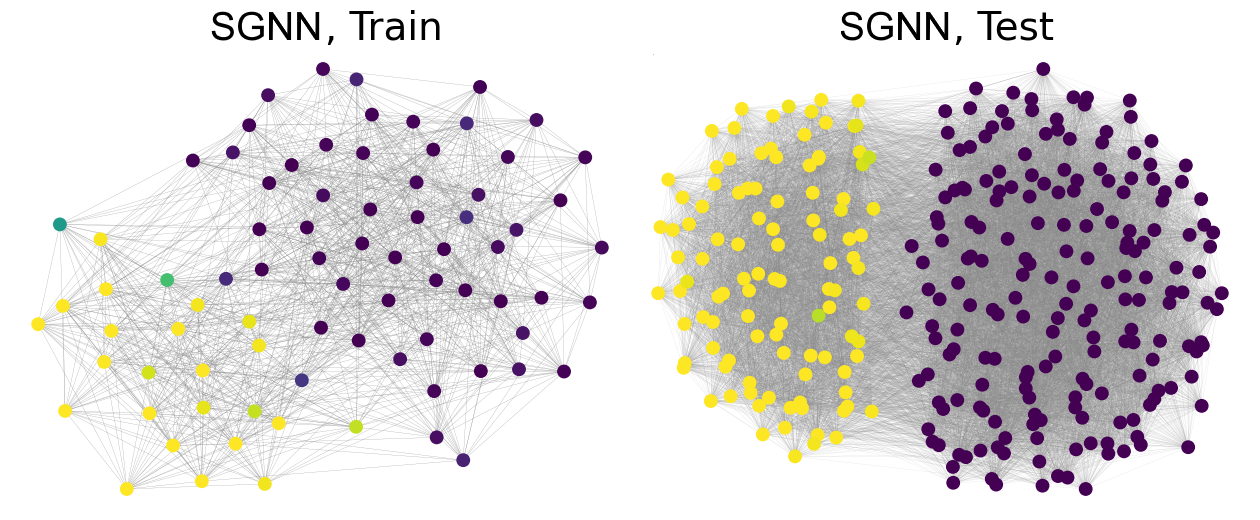}
    \caption{\small Illustration of Prop.~\ref{prop:sbm-eq} (Sec.~\ref{sec:eq-example}). On an SBM with constant degree function, a GNN (top) might overfit the training set, but converges to a constant c-GNN. On the contrary, there exists a c-SGNN (bottom) that perfectly separates the communities. Details can be found in App.~\ref{app:numerics}.}\label{fig:sbm}
    \vspace{-20pt}
\end{wrapfigure}

\paragraph{Contribution and outline.} In this paper, we study the convergence of SGNNs on large random graphs towards ``c-SGNNs'', and analyze the approximation power of c-GNNs and c-SGNNs. After some preliminary results in Sec.~\ref{sec:preliminaries}, we study the convergence of SGNNs in Sec~\ref{sec:conv}. In Sec.~\ref{sec:approx-inv} and \ref{sec:approx-eq}, we show that c-SGNNs are strictly more powerful than c-GNNs in both permutation-invariant and equivariant case. We then prove the universality of c-SGNNs on several popular models of random graphs, including Stochastic Block Models (SBMs) and radial kernels. For instance, we show in Sec.~\ref{sec:eq-example} that c-SGNNs can perfectly identify the communities of most SBMs, including some for which any c-GNN provably fails (Fig.~\ref{fig:sbm}).

\paragraph{Related work.} The approximation power of GNNs (on finite graphs) has been a topic of great interest in recent years, and we do not attempt to make an exhaustive list of all results here. The landmark paper \cite{Xu-WL} showed that permutation-invariant GNNs were at best as powerful as the WL test, and later models \cite{Maron-powerful, Bruna-GI, Geerts-kWL} were constructed to be as powerful as higher-order WL tests, using for instance higher-order tensors \cite{Maron-powerful}. The link between the graph isomorphism problem and function approximation was made in \cite{Bruna-GI, Keriven-universal} and extended in \cite{Lelarge-power}. Some architectures were proven to be universal \cite{Maron-universal, Keriven-universal} but involve tensors of unbounded order. When graphs are equipped with unique node identifiers, the approximation power of GNNs can be significantly improved if relaxing permutation-invariance/equivariance \cite{Loukas-depth, Loukas-GI}. Recent works have proposed methods to restore these properties \cite{Vignac-smp}. As mentioned above, the resulting SGNNs are indeed more powerful than the WL test \cite{Vignac-smp}, and even universal on graphs with bounded degrees when allowing powerful layers. To our knowledge, the approximation power of SGNNs as they are implemented in practice is still open. We treat of their continuous limit in this paper.

Fewer results can be found in the large-graph limit. Beyond the graph-isomorphism paradigm, authors have studied the capacity of GNNs to count graph substructures \cite{Chen-substructure, Barabasi-subgraph} or identify various graph properties \cite{Garg-generalization, Maron-generalization}. Several works have studied the large-graphs limit of GNNs \cite{Ruiz-transferability, Magner-power, Levie-transferability, us}, assuming random graph or graphon models. The degree function has been proven to be a crucial element for the discriminative power of c-GNNs, and they will not be able to distinguish graphons with close degree functions \cite{Magner-power}. Here we show how c-SGNNs allow to overcome these limitations, and provide the first universality theorems for (S)GNNs in the continuous limit. Finally, we note that universal architectures exist on \emph{measures} \cite{DeBie-measure-nn,Zweig-measure-nn}, which can be seen as limits of \emph{point clouds}, but, similar to the discrete case \cite{Zaheer-deep-sets, Maron-universal, Keriven-universal}, the graph case may be significantly harder to study.

\section{Preliminaries}\label{sec:preliminaries}

In this section, we group notations on random graphs and GNNs, present the convergence result of \cite{us} slightly adapted to our context, and introduce the SGNNs of \cite{Vignac-smp}. An undirected graph $G$ with $n$ nodes is represented by a symmetric adjacency matrix $A \in \RR^{n\times n}$. It is isomorphic to all graphs with adjacency matrices $\sigma A \sigma^\top$, for any permutation matrix $\sigma \in \{0,1\}^{n\times n}$. Matrix rows (resp. columns) are denoted by $A_{i,:}$ (resp. $A_{:,i}$). The norm $\norm{\cdot}$ is the operator norm for matrices and Euclidean norm for vectors.
For a compact metric space $\Xx$, we denote by $\Cc(\Xx,\RR)$ the space of bounded continuous functions $\Xx\to \RR$, equipped with $\norm{f}_\infty = \sup_x \abs{f(x)}$. For a multivariate function $f \in \Cc(\Xx,\RR^d)$, $f_i$ denotes its i$th$ coordinate, and $\norm{f}_\infty = (\sum_i \norm{f_i}_\infty^2)^{1/2}$.

\subsection{Random graph model}\label{sec:RG}

We consider ``latent position'' random graphs, which include several popular models such as SBMs or graphons \cite{Lovasz-graphon}. The nodes are associated with unobserved latent random variables $x_i$ drawn \emph{i.i.d}. For the edges, we will examine two cases: an ``ideal'' one with deterministic weighted edges, and a more ``realistic'' one where the edges are randomly drawn independently. In the literature, the latter is often considered as an idealized framework to study the behavior of various algorithms \cite{VonLuxburg-SC, Rosasco-ope}.

Let $(\Xx, m_\Xx)$ be a compact metric space that is not a singleton, and assume that the $\varepsilon$-covering numbers\footnote{the number of balls of radius $\varepsilon$ required to cover $\Xx$} of $\Xx$ scale as $\order{\varepsilon^{-d_\Xx}}$ for some dimension $d_\Xx$. We denote by $\Ww$ the set of symmetric bivariate functions in $\Cc(\Xx \times \Xx, [0,1])$ that are $L_W$-Lipschitz in each variable, and by $\Pp$ the set of probability distributions over $\Xx$ equipped with the total variation norm $\normTV{\cdot}$. A graph with $n$ nodes is generated according to a random graph model $(W,P) \in \Ww \times \Pp$ as follows:
\begin{equation*}
    x_1,\ldots,x_n \stackrel{\text{\emph{i.i.d}}}{\sim} P, \qquad a_{ij} =
    \begin{cases}
        W(x_i, x_j) &\text{deterministic edges case} \\
        \alpha_n^{-1}\text{Bernoulli}(\alpha_n W(x_i, x_j)) &\text{random edges case}
    \end{cases}
\end{equation*}
where $\alpha_n$ is the \emph{sparsity level} of the graph in the random edges case, which we assume to be known for simplicity. When $\alpha_n \sim 1$, the graph is said to be \emph{dense}, when $\alpha_n \sim \frac{1}{n}$ the graph is \emph{sparse}, and when $\alpha \sim \frac{\log n}{n}$ the graph is \emph{relatively sparse}. Note that we have normalized the Bernoulli edges by $1/\alpha_n$ such that $\EE a_{ij} = W(x_i, x_j)$ conditionally on $x_i,x_j$. When $\Xx$ is finite, the model is called an SBM (see Sec.~\ref{sec:example-inv}).

Like finite graphs, random graph models can be isomorphic to one another \cite{us, Lovasz-graphon}. In this paper, similar to \cite{us} we consider that, for any bijection $\varphi:\Xx \to \Xx$, the model $(W,P)$ is isomorphic to $(W_\varphi, \varphi^{-1}_\sharp P)$, where $W_\varphi(x,y) = W(\varphi(x), \varphi(y))$ and $f_\sharp P$ is the pushforward of $P$ (the distribution of $f(X)$). Indeed, it is easy to see that both produce exactly the same distribution over graphs.

\subsection{Graph Neural Networks}\label{sec:gnn}

Following \cite{us, Defferrard-cheb}, in this paper we consider the so-called ``spectral'' version of GNNs, which include several message-passing models for certain aggregation functions. We consider polynomial filters $h$ defined as $h(A) = \sum_k \beta_k A^k$ for a matrix or operator $A$. In practice, the order of the filters is always finite, but our results are valid for infinite-order filters (assuming that the sum always converges for simplicity). We consider any activation function $\rho:\RR \to \RR$ which satisfies $\rho(0)=0$ and $\abs{\rho(x)-\rho(y)}\leq \abs{x-y}$ for which the universality theorem of MLPs applies \cite{Pinkus-MLP}, e.g. ReLU.

Spectral GNNs are defined by successive filtering of a graph signal. Given an input signal $Z^{(0)}\in \RR^{n\times d_0}$, at each layer $\ell=0, \ldots, M-1$:
\begin{equation}\label{eq:gnn1}
\mathsmaller{Z^{(\ell+1)}_{:,j} = \rho\left(\sum_{i=1}^{d_\ell} h_{ij}^{(\ell)}\left(\tfrac{1}{n}A\right) Z^{(\ell)}_{:,i} + b_j^{(\ell)} 1_n\right) \in \RR^n \qquad j=1,\ldots, d_{\ell+1}}\, ,
\end{equation}
where $h_{ij}^{(\ell)}(A) = \sum_k \beta_{ijk}^{(\ell)} A^k$ are trainable graph filters, $b_j^{(\ell)}\in \RR$ are trainable additive biases and $\rho$ is applied pointwise. We note the normalization $A/n$ that will be necessary for convergence.
GNNs exist in two main versions: so-called ``permutation-invariant'' GNNs $\bar \Phi$ output a single vector for the entire graph, while ``permutation-equivariant'' GNNs $\Phi$ output a graph signal:
\begin{equation}\label{eq:gnn2}
    \mathsmaller{\Phi_A(Z^{(0)}) = g\pa{Z^{(M)}} \in \RR^{n \times d_{out}} \qquad \bar\Phi_A(Z^{(0)}) = g\pa{\tfrac{1}{n} \sum_{i=1}^n Z^{(M)}_{i,:}} \in \RR^{d_{out}}}\ ,
\end{equation}
where $g:\RR^{d_M} \to \RR^{d_{out}}$ is an MLP applied row-wise in the equivariant case. By construction, such GNNs are indeed equivariant or invariant to node permutation: for all permutation matrices $\sigma$,
we have $
    \Phi_{\sigma A \sigma^\top}(\sigma Z^{(0)}) = \sigma \Phi_{A}(Z^{(0)})$ and $\bar\Phi_{\sigma A \sigma^\top}(\sigma Z^{(0)}) = \bar\Phi_{A}(Z^{(0)})$.

\paragraph{Continuous GNNs.} In \cite{us}, the authors show that, in the large random graphs limit, GNNs converge to the following ``continuous'' models (coined c-GNNs), which propagate \textbf{functions over the latent space} $f^{(\ell)} \in \Cc(\Xx, \RR^{d_\ell})$ instead of graph signals. Given an input function $f^{(0)} \in \Cc(\Xx, \RR^{d_0})$:
\begin{equation*}
    \mathsmaller{f^{(\ell + 1)}_j = \rho \left( \sum_{i=1}^{d_\ell} h_{ij}^{(\ell)}(T_{W,P})f^{(\ell)}_i + b^{(\ell)}_j \right)}\, ,
\end{equation*}
where $T_{W,P}$ is the operator $T_{W,P}f = \int W(\cdot,x)f(x)dP(x)$. Then similar to \eqref{eq:gnn2} the permutation-equivariant/invariant versions of c-GNNs are defined as:
\begin{equation*}
    \mathsmaller{\Phi_{W,P}(f^{(0)}) = g \circ f^{(M)}, \qquad
    \bar\Phi_{W,P}(f^{(0)}) = g\pa{ \int f^{(M)}(x) dP(x)}}\, .
\end{equation*}
Remark that $\Phi_{W,P}(f^{(0)})$ is itself a \emph{function} in $\Cc(\Xx,\RR^{d_{out}})$ while $\bar\Phi_{W,P}(f^{(0)}) \in \RR^{d_{out}}$ is a vector. Moreover, by virtue of the polynomial filters, the four architectures $\Phi_A, \bar \Phi_A, \Phi_{W,P}, \bar \Phi_{W,P}$ have the exact same set of parameters.
Like in the discrete case, one can prove that c-GNN are equivariant or invariant to isomorphisms of random graph models~\cite{us}: for all bijection $\varphi:\Xx \to \Xx$, we have $\Phi_{W_\varphi, \varphi^{-1}_\sharp P}(f\circ \varphi) = \Phi_{W,P}(f) \circ \varphi$ and $\bar\Phi_{W_\varphi, \varphi^{-1}_\sharp P}(f\circ \varphi) = \bar\Phi_{W,P}(f)$.

Let us now turn to convergence of GNNs to c-GNNs. The following result is adapted\footnote{The authors in \cite{us} proved this for the normalized Laplacian, which allows bypassing the knowledge of $\alpha_n$. Here we use the adjacency matrix, for our later results on approximation power.} from \cite{us}. While the outputs of permutation-invariant GNNs and c-GNNs are vectors in $\RR^{d_{out}}$ that can be directly compared, the output graph signal of a permutation-equivariant GNN is compared with a \emph{sampling} of the output function of the corresponding c-GNN: for a graph signal $Z\in \RR^{n \times d}$, a function $f:\Xx \to \RR^d$ and $X=\{x_i\}_{i=1}^n$, we define the (square root of the) mean-square error as $\MSE_X(Z,f) = ( \frac{1}{n}\sum_{i=1}^n \norm{Z_i - f(x_i)}_2^2 )^{1/2}$. The proof of Theorem \ref{thm:conv-gnn} with all multiplicative constants is given in App.~\ref{app:conv-gnn}. Recall that $d_\Xx$ is the ``dimension'' of $\Xx$.

\begin{theorem}\label{thm:conv-gnn}
    Assume $G$ is drawn from $(W,P)$ and has latent variables $X$. Fix $\rho, \nu>0$.
    \setlist{nolistsep}
    \begin{itemize}[leftmargin=15pt, noitemsep]
        \item \textbf{In the deterministic edges case:} with probability $1-\rho$, for \emph{all} $Z^{(0)}$:
        \begin{equation}\label{eq:conv-gnn-eq-det}
            \MSE_X(\Phi_A(Z^{(0)}), \Phi_{W,P}(f^{(0)})) \leq
            C \cdot \MSE_X(Z^{(0)}, f^{(0)}) + R_1(n)
        \end{equation}
        for some constant $C$ and $R_1(n) = \order{\sqrt{(d_\Xx + \log(1/\rho))/n}}$.
        \item \textbf{In the random edges case:} assume that the sparsity level is $\alpha_n \gtrsim n^{-1}\log n$. There is a constant $C_\nu$ such that, with probability $1-\rho-n^{-\nu}$, for \emph{all} $Z^{(0)}$:
        \begin{equation}\label{eq:conv-gnn-eq-rand}
            \MSE_X(\Phi_A(Z^{(0)}), \Phi_{W,P}(f^{(0)})) \leq
            C \cdot \MSE_X(Z^{(0)}, f^{(0)}) + R_1(n) + R_2(n)
        \end{equation}
        where $R_2(n) = \order{C_\nu/\sqrt{\alpha_n n}}$.
        \item \textbf{In the permutation-invariant case:} The exact same results hold for $
            \norm{\bar \Phi_A(Z^{(0)}) - \bar \Phi_{W,P}(f^{(0)})}
        $ instead of the MSE on the left-hand-side, with an added error term $R_3(n) = \order{\sqrt{\log(1/\rho)/n}}$.
    \end{itemize}
\end{theorem}

By the theorem above, a GNN converges to its continuous counterpart if $Z^{(0)}$ is (close to) a sampling of a function $f^{(0)}$ at the latent variables. This is directly assumed in \cite{us}. In the present paper however, we do not suppose that input node features are available. While several strategies have been proposed in the literature, a popular baseline is to simply take constant input $Z^{(0)} = 1_n$ (which, by a multiplication by $A$ on the first layer, is also equivalent to inputing the degrees as in \cite{Bruna-clust} for instance). In this case, there is convergence to a c-GNN with $f^{(0)}=1$. For simplicity in the rest of the paper we drop the notation ``$(1)$'' and write $\Phi_A = \Phi_A(1_n)$, $\Phi_{W,P} = \Phi_{W,P}(1)$, and so on.
It is known that GNNs are limited on regular graphs: for permutation-invariant GNNs, the WL test cannot distinguish regular graphs of the same order, and for permutation-equivariant GNNs, $\Phi_A(1_n)$ is \emph{constant over the nodes}. Similarly for c-GNN, the degree function $\int W(\cdot,x)dP(x)$ is key in the discriminative power of c-GNNs \cite{Magner-power}, and if it is constant, then $\Phi_{W,P}(1)$ is a constant function (see Fig.~\ref{fig:sbm}).

\subsection{SGNN: GNN with unique node identifiers}\label{sec:sgnn}

To remedy the absence of input node features, in \cite{Vignac-smp} the authors propose an architecture with \emph{unique node identifiers} while still respecting permutation invariance/equivariance. More precisely, they first choose an \emph{arbitrary} ordering of the nodes $q=1,\ldots,n$, apply a GNN to each \emph{one-hot vector} $e_q = [0,\ldots, 1, \ldots, 0]$, and restore equivariance to permutation by a final pooling.
For later purpose of convergence, we generalize this strategy to any collection $E_q(A) \in \RR^{n \times d_0}$ that satisfies:
\begin{equation}\label{eq:hyp-Eq}
    E_q(\sigma A\sigma^\top) = \sigma E_{\sigma^{-1}(q)}(A)\, .
\end{equation}
where by an abuse of notation $\sigma(q)$ designates the permutation function applied to index $q$. For instance, it can be any filtering of one-hot vector $E_q(A) = h(A) e_q$. After the GNN $\Phi_A$, a pooling is applied to restore equivariance, then a second GNN $\Phi'_A$ that is either invariant or equivariant:
\begin{equation}\label{eq:sgnn}
    \mathsmaller{\Psi_A = \Phi'_A\pa{\tfrac{1}{n} \sum_q \Phi_A\pa{E_q(A)}} \in \RR^{n \times d'_{out}}, \qquad \bar \Psi_A=\bar \Phi'_A\pa{\frac{1}{n} \sum_q \Phi_A\pa{E_q(A)}} \in \RR^{d'_{out}}}
\end{equation}
In \cite{Vignac-smp}, these architectures, called SMPs, are interpreted as doing message-passing over matrices, and can use more general pooling and aggregation functions. In a sense, what we call SGNN are ``spectral'' versions of SMP, but are essentially the same idea. It is not difficult to see that SGNNs satisfy: $\Psi_{\sigma A \sigma^\top} = \sigma \Psi_A$ and $\bar\Psi_{\sigma A \sigma^\top} =  \bar \Psi_A$.

\section{Convergence of SGNNs on large random graphs}\label{sec:conv}

In this section, we extend Theorem \ref{thm:conv-gnn} to SGNNs.
To define continuous SGNNs, we consider a \emph{bivariate} input function $\eta_{W,P}:\Xx \times \Xx \to \RR^{d_0}$ such that: the mapping $(W,P) \mapsto \eta_{W,P}$ is continuous\footnote{for the norm $\norm{\cdot}_\infty + \normTV{\cdot}$ on $\Ww \times \Pp$}, for any $(W,P)$, $\eta_{W,P}$ is $C_\eta$-bounded and $L_\eta$-Lipschitz in each variable, and similar to \eqref{eq:hyp-Eq} is respects the following:
\begin{equation}\label{eq:hyp-eta}
    \mathsmaller{\eta_{W_\varphi,\varphi^{-1}_\sharp P}(x,y) = \eta_{W,P}(\varphi(x),\varphi(y))}\, ,
\end{equation}
for all bijections $\varphi : \Xx \to \Xx$. For instance, $\eta_{W,P}=W$ or any filter $\eta_{W,P}(x,y)=[h(T_{W,P})W(\cdot,y)](x)$ satisfy these conditions. A c-SGNN is then defined as:
\begin{equation}\label{eq:csgnn}
    \mathsmaller{\Psi_{W,P} = \Phi'_{W,P}\pa{\int \Phi_{W,P}(\eta_{W,P}(\cdot, x)) dP(x)}}\, ,
\end{equation}
for the equivariant case, or similarly $\bar \Psi_{W,P} = \bar \Phi'_{W,P}\pa{\int \Phi_{W,P}(\eta_{W,P}(\cdot, x)) dP(x)}$ for the invariant case.
Again, using the properties of c-GNNs, it is easy to see that c-SGNNs respect random graphs isomorphism:
$
    \Psi_{W_\varphi,\varphi^{-1}_\sharp P} =\Psi_{W,P} \circ \varphi  $ and $\bar\Psi_{W_\varphi,\varphi^{-1}_\sharp P} =  \bar \Psi_{W,P}
$.
We are now ready to extend Theorem \ref{thm:conv-gnn}. The proof and multiplicative constants are in App.~\ref{app:conv-sgnn}.

\begin{theorem}\label{thm:conv-sgnn}
    Assume $G$ is drawn from $(W,P)$ with latent variables $X$. Fix $\rho, \nu>0$.
    \setlist{nolistsep}
    \begin{itemize}[leftmargin=15pt, noitemsep]
        \item \textbf{In the deterministic edges case:} with probability $1-\rho$:
        \begin{equation}\label{eq:conv-sgnn-eq-det}
            \MSE_X(\Psi_A, \Psi_{W,P}) \leq
            C' \cdot \sup_q \MSE_X(E_q(A), \eta(\cdot,x_q)) + R'_1(n)
        \end{equation}
        for some constant $C'$ and $R'_1(n) = \order{\sqrt{(d_\Xx + \log(1/\rho))/n}}$.
        \item \textbf{In the random edges case:} assume that the sparsity level is $\alpha_n \gtrsim n^{-1}\log n$. There is a constant $C_\nu$ such that, with probability $1-\rho-n^{-\nu}$:
        \begin{equation}\label{eq:conv-sgnn-eq-rand}
            \MSE_X(\Psi_A, \Psi_{W,P}) \leq
            C' \cdot \sup_q \MSE_X(E_q(A), \eta(\cdot,x_q)) + R'_1(n) + R'_2(n)
        \end{equation}
        where $R'_2(n) = \order{C_\nu/\sqrt{\alpha_n n}}$.
        \item \textbf{In the permutation-invariant case:} The exact same results hold for $
        \norm{\bar \Psi_A - \bar \Psi_{W,P}}
        $ instead of the MSE, with an added error term $R'_3(n) = \order{\sqrt{\log(1/\rho)/n}}$.
    \end{itemize}
\end{theorem}

Hence, we obtain convergence when the input signal $E_q(A)$ is close to being a sampling of a function $\eta_{W,P}$ at $x_q$. The choice of input signal/function is therefore crucial. Let us examine a few strategies.

\paragraph{One-hot vectors.} If one chooses one-hot vectors $E_q = e_q$ as in \cite{Vignac-smp}, then we can see that the SGNN converges to the continuous architecture with input $\eta_{W,P}=0$, since $\MSE_X(e_q, 0) \to 0$. Since $\Psi_{W,P}$ is nothing more than a traditional c-GNN $\Phi_{W,P}$ with constant input in that case, this is not a satisfying choice in terms of approximation power.

\paragraph{One-hop filtering.} If we choose to ``filter'' $e_q$ once and take $E_q(A) = Ae_q$, the natural continuous equivalent is $\eta_{W,P}=W$. Such a strategy only works for deterministic edges: indeed,
\begin{align*}
    \MSE_X(Ae_q, W(\cdot,x_q)) &= \mathsmaller{\left(n^{-1}\sum_{i=1}^n (a_{iq} - W(x_i,x_q))^2\right)^{1/2}} \\
    &\begin{cases}
        = 0 &\text{with deterministic edges}\\
        \approx \sqrt{n^{-1} \sum_i Var(a_{iq})} \sim \alpha^{-1} &\text{with random edges, w.h.p.}
    \end{cases}
\end{align*}
where the last line comes from a simple application of Hoeffding's inequality. Hence, in the case of random edges, the MSE does not vanish and typically diverges for non-dense graphs (Fig.~\ref{fig:conv}).

\begin{wrapfigure}{r}{0.35\textwidth}
    \centering
    \vspace{-20pt}
    \includegraphics[width=0.35\textwidth]{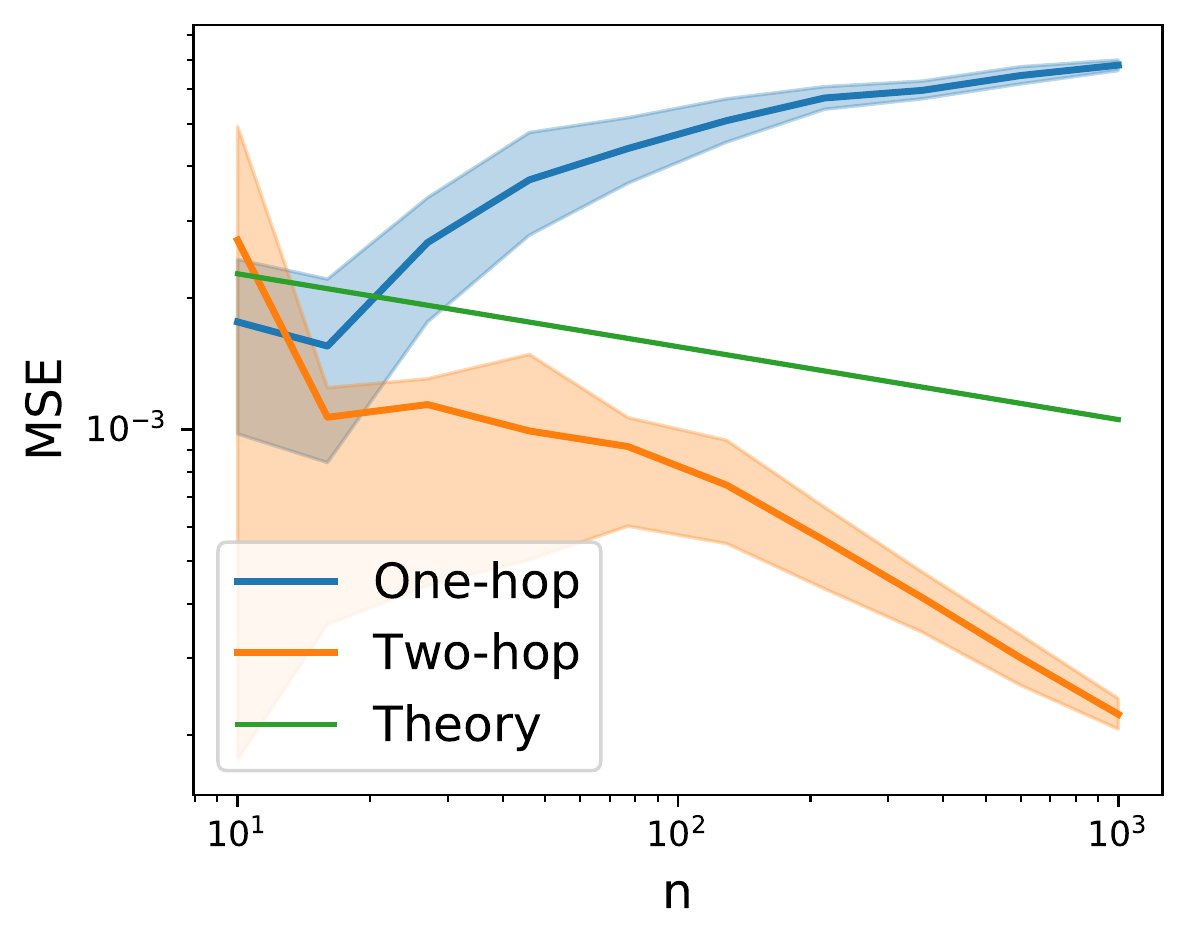}
    \caption{\small Difference between the outputs of some $\bar \Psi_{A}$ in the random edges case with $\alpha_n\sim n^{-1/3}$ and in the deterministic edges case, which converges to $\bar \Psi_{W,P}$. Convergence is observed for two-hop filtering only. The theoretical rate of Prop.~\ref{prop:two-hop-input} is slightly pessimistic. Details can be found in App.~\ref{app:numerics}.}
    \label{fig:conv}
    \vspace{-30pt}
\end{wrapfigure}
\paragraph{Two-hop filtering.} We can therefore choose to filter \emph{twice} and consider $\eta_{W,P}(x,y) = T_{W,P}[W(\cdot,y)](x)$, that is, $E_q(A) = A^2 e_q / n$. We have the following result.

\begin{proposition}\label{prop:two-hop-input}
    In the random edges case, with probability $1-\rho$, we have for all $q$:
    \begin{equation*}
        \mathsmaller{\textup{MSE}_X\pa{\frac{A^2 e_q}{n}, T_{W,P}[W(\cdot,x_q)](\cdot)} \lesssim \order{\frac{\sqrt{\log(n/\rho)}}{\alpha_n \sqrt{n}}}}\, .
    \end{equation*}
    In the deterministic edges case, the rate is $\order{1/\sqrt{n}}$.
\end{proposition}
Convergence is guaranteed when $\sqrt{\log n}/(\alpha_n \sqrt{n}) \to 0$, which is stronger than relative sparsity. The difference between one-hop and two-hop filtering is illustrated in Fig.~\ref{fig:conv}.
Other strategies may also be examined, and we leave this for future work. For instance, based on Theorem \ref{thm:lei} in App.~\ref{app:add}, a strategy based on the eigenvectors of $A$ could lead to convergence even in the relatively sparse case. In the rest of the paper, we explicitly write when our results are valid for one- or two-hop filtering.

\section{Approximation power: permutation-invariant case}\label{sec:approx-inv}

In this section, we study the approximation power of continuous architectures in the permutation-invariant case. We seek to characterize the functions $\Ww \times \Pp \to \RR$ that can be well-approximated by a c-GNN or c-SGNN. We derive a generic criterion for universality arising from the Stone-Weierstrass theorem, before proving that c-SGNNs are indeed strictly more powerful than c-GNNs. Finally, we give several examples of models for which c-SGNNs are universal.

\subsection{Generic result with Stone-Weierstrass theorem}

A classical tool to prove universality of neural nets in the literature is the Stone-Weierstrass theorem \cite{Hornik-mlp, Keriven-universal}, which states that an algebra of functions that separates points on a compact space is dense in the set of continuous functions. Although NNs are typically not an algebra since they are not closed by multiplication, one of the original proofs of universality for MLPs solves this by a clever trick \cite{Hornik-mlp} which we recall in Lemma~\ref{lem:SW} in App.~\ref{app:add}. A direct application results in the following proposition.
\begin{proposition}\label{prop:sw-inv}
    Let $\Mm \subset \Ww \times \Pp$ be a \textbf{compact} subset of $\Ww \times \Pp$. c-SGNNs (resp. c-GNNs) are dense in $\Cc(\Mm,\RR)$ if and only if: for all $(W,P) \neq (W',P') \in \Mm$, there is a c-SGNN (resp. a c-GNN) such that $\bar \Psi_{W,P} \neq \bar \Psi_{W',P'}$ (resp. $\bar \Phi_{W,P} \neq \bar \Phi_{W',P'}$).
\end{proposition}
Note that, by construction of the permutation-invariant c-(S)GNNs, universality is only possible when $\Mm$ does \emph{not} contain two isomorphic versions of the same random graph model. Equivalently, $\Mm$ may be a larger set \emph{quotiented} by random graph isomorphism.

\subsection{c-SGNNs are more powerful than c-GNNs}

SGNNs were proven to be strictly more powerful than the WL test, and therefore than GNNs, in \cite{Vignac-smp}. In the theorem below, we check that this strict inclusion holds for their continuous limits.

\begin{theorem}\label{thm:gnn-vs-sgnn-inv}
    The set of functions of the form $(W,P) \to \bar \Phi_{W,P}$ is \textbf{strictly} included in the set of functions $(W,P) \to \bar \Psi_{W,P}$, for both one- and two-hop input filtering.
\end{theorem}

This theorem is proven by constructing two models that are distinguished by a c-SGNN but not by any c-GNN. We note that there might be subsets $\Mm \subset \Ww \times \Pp$ that do not contain such models, and therefore on which c-GNNs and c-SGNNs have the same approximation power.

\subsection{Examples}\label{sec:example-inv}

While a generic universality theorem on random graphs seems to be out-of-reach for the moment, we examine several interesting examples, and pave the way for future extensions. We focus on c-SGNNs, but, sometimes, may not conclude on the power of c-GNNs: we could not prove that they are universal, but were not able to find a counter-example either. For simplicity, and since our purpose here is mainly illustrative, we mostly focus on one-hop input filtering.

\paragraph{Stochastic Block Models.}
SBMs \cite{Holland-SBM} are classical models to emulate graphs with communities. In our settings, SBMs with $K$ communities can be obtained with a finite latent space $\abs{\Xx} = K$, typically $\Xx = \{1, \ldots, K\}$. The kernel $W(k,k') = W_{kk'}$ can be represented as a matrix $W \in S_K$, where $S_K$ is the set of symmetric matrices in $[0,1]^{K \times K}$, and the distribution $P(k) = P_k$ as a vector $P\in \Delta^{K-1}$, where $\Delta^{K-1} = \{P\in [0,1]^K, \sum_k P_k=1\}$ is the $(K\text{--}1)$-dimensional simplex.

In the following proposition, we fix $P$ and examine universality with respect to $W$. In this case, continuous GNNs are actually quite similar to discrete ones on matrices $W$, except that the probability vector $P$ also intervenes in the computation. While GNNs on finite graphs can only be universal when using high-order tensors \cite{Maron-universal, Keriven-universal} due to invariance to graph-isomorphism, here $P$ can help to disambiguate this constraint.
We will say that $P \in \Delta^{K-1}$ is \textbf{incoherent} if: for signs $s \in \{-1,0,1\}^k$, having $\sum_{k=1}^K s_k P_k = 0$ implies $s = 0$. That is, no probability is an exact sum or difference of the others. We note that a vector drawn uniformly on $\Delta^{K-1}$ is incoherent with probability $1$. For incoherent probability vectors, we can show universality of c-SGNNs. Moreover, this is actually a case where we \emph{can} prove that c-GNNs are, in turn, \emph{not} universal.

\begin{proposition}\label{prop:sbm-inv}
    For one-hop input filtering, if $P$ is incoherent the space of functions $W \to \bar \Psi_{W,P}$ is dense in $\Cc(S_K,\RR)$. Moreover, there exists $P$ incoherent and $W \neq W'$ such that, for any c-GNN, $\bar \Phi_{W,P} = \bar \Phi_{W',P}$.
\end{proposition}

\paragraph{Additive kernel.}
Let us now fix $W$ and examine universality with respect to $P$. A classical theorem on symmetric continuous functions \cite{Zaheer-deep-sets} states that any $W$ can be arbitrarily well approximated as $W(x,y) \approx u(v(x)+v(y))$ for some functions $u,v$.
    Inspired by this result, a kernel will be said to be \textbf{additive} if it can (exactly) be written as $W(x,y) = u(v(x) + v(y))$, and \textbf{injectively additive} if both $u,v$ are \emph{continuous and injective}. We prove universality in the unidimensional case below.

\begin{proposition}\label{prop:decomp-inv}
    Let $\Xx \subset \RR$, and $\tilde \Pp$ be any compact subset of $\Pp$. Assume $W$ is injectively additive with $\textup{Im}(v) \subset \RR$. For one-hop filtering, the space of functions $P \to \bar \Psi_{W,P}$ is dense in $\Cc(\tilde \Pp,\RR)$.
\end{proposition}

It is easy to see that injectively additive kernels include all SBMs for which $W_{ij} \neq W_{ij'}$ when $j\neq j'$, so Prop.~\ref{prop:decomp-inv} completes Prop.~\ref{prop:sbm-inv}.
However, unlike additive kernels, injectively additive kernels are \emph{a priori} \textbf{not} universal approximators of symmetric continuous functions: this result \cite{Zaheer-deep-sets} is only valid when $\Im(v)$ can be multidimensional. But it is known for instance that there is no continuous injective map from $[0,1]^2$ to $[0,1]$, so if $\Im(v) = [0,1]^2$, $u$ cannot be both continuous and injective.

\paragraph{Radial kernel.}
We conclude this section with an important class of kernels, \emph{radial} kernels $W(x,y) = w(\norm{x-y})$ for some function $w:\RR_+ \to [0,1]$. They include the popular Gaussian kernel and so-called $\varepsilon$-graphs. Below, we give an example in one dimension, for which c-SGNNs are universal on symmetric distributions. The case of non-symmetric distributions seems more involved and we leave it for future investigations.

\begin{proposition}\label{prop:rad-inv}
    Assume that $\Xx = [-1,1]$ and $W(x,y) = w(\abs{x-y})$ where $w$ is continuous and injective. Let $\tilde \Pp \subset \Pp$ be any compact set of \textbf{symmetric} distributions. For one-hop input filtering, the space of functions $P \to \bar \Psi_{W,P}$ is dense in $\Cc(\tilde \Pp,\RR)$.
\end{proposition}

%%%%%%%%%%%%

\section{Approximation power: permutation-equivariant case}\label{sec:approx-eq}

In the equivariant case, recall that the outputs of c-(S)GNNs are \emph{functions} on $\Xx$. The ``traditional'' notion of universality is to evaluate the approximation power of mappings $(W,P) \to \Ff$, where $\Ff$ is some space of equivariant functions, as is done for the discrete case in \cite{Keriven-universal, Lelarge-power}. However, a potentially simpler and more relevant notion here is to \emph{fix} $(W,P)$, and directly examine the properties of the space of functions $\Xx \to \RR$ represented by c-(S)GNNs, that is, the space of functions $\{\Psi_{W,P}:\Xx \to \RR\}$ for all possible c-SGNNs $\Psi_{W,P}$ (and similar for $\Phi_{W,P}$). Indeed, this directly answers such questions as: given an SBM, does there exist a c-GNN that can labels the communities (Fig.~\ref{fig:sbm})? Or: given the structure of a mesh, what functions can be computed on it, e.g. for segmentation?

\subsection{Generic result with Stone-Weierstrass theorem}

As in the invariant case, the Stone-Weierstrass theorem yields a generic separation condition.

\begin{proposition}\label{prop:sw-eq}
    Let $(W,P)$ be fixed. Then c-SGNNs (resp. c-GNNs) are dense in $\Cc(\Xx,\RR)$ iff: for all $x \neq x' \in \Xx$, there is a c-SGNN (resp. a c-GNN) such that $\Psi_{W,P}(x) \neq \Psi_{W,P}(x')$ (resp. $\Phi_{W,P}(x) \neq \Phi_{W,P}(x')$).
\end{proposition}

If $(W,P)$ are such that c-(S)GNNs satisfy some symmetry or invariance (see for instance Prop.~\ref{prop:rad-eq}), then the space $\Xx$ can be quotiented to obtain universality among functions satisfying these constraints.

\subsection{c-SGNNs are more powerful than c-GNNs}

Using a proof similar to Theorem \ref{thm:gnn-vs-sgnn-inv}, we can then prove that c-SGNNs are indeed strictly more powerful than c-GNNs for some $(W,P)$.

\begin{theorem}\label{thm:gnn-vs-sgnn-eq}
    For both one- and two-hop input filtering, the following holds. For any $(W,P)$, the set of functions of the form $\Phi_{W,P}$ is included in the set of functions $\Psi_{W,P}$, and there exist $(W,P)$ such that the inclusion is \textbf{strict}.
\end{theorem}

Again, we note that, for some random graph models $(W,P)$, c-SGNNs and c-GNNs might have the same approximation power. %However, it is very easy to construct a simple model on which any c-GNN will provably fail while c-SGNNs are universal (see Prop.~\ref{prop:sbm-eq} below and Fig.~\ref{fig:sbm}).

\subsection{Examples}\label{sec:eq-example}

We treat the same examples as before, with the addition of two-hop filtering for SBMs and radial kernels on the $d$-dimensional sphere.

\paragraph{SBM.}
Universality in the SBM case corresponds to being able to distinguish communities, that is, $\Psi_{W,P}:\Xx \to \RR$ returns a different value for each element of the latent space $\Xx$. In the following result, we assume that $W$ is invertible, and prove the result for \emph{both} one- and two-hop filtering, meaning that c-SGNNs can indeed distinguish communities of SBMs with random edges under some mild conditions. On the other hand, c-GNNs may fail on such models.

\begin{proposition}\label{prop:sbm-eq}
    Let $P$ be incoherent, and $W$ be invertible. For \textbf{both} one- and two-hop input filtering, c-SGNNs are dense in $\Cc(\Xx, \RR)$. Moreover, there exist $(W,P)$ satisfying the conditions above such that any c-GNN $\Phi_{W,P}$ is a constant function.
\end{proposition}

This proposition is illustrated in Fig.~\ref{fig:sbm}: on an SBM with constant degree function, any GNN converges to a c-GNN with constant output, while a SGNN with two-hop input filtering can be close to a c-SGNN that perfectly separates the communities.

\begin{figure}
    \centering
    \includegraphics[width=0.32\textwidth]{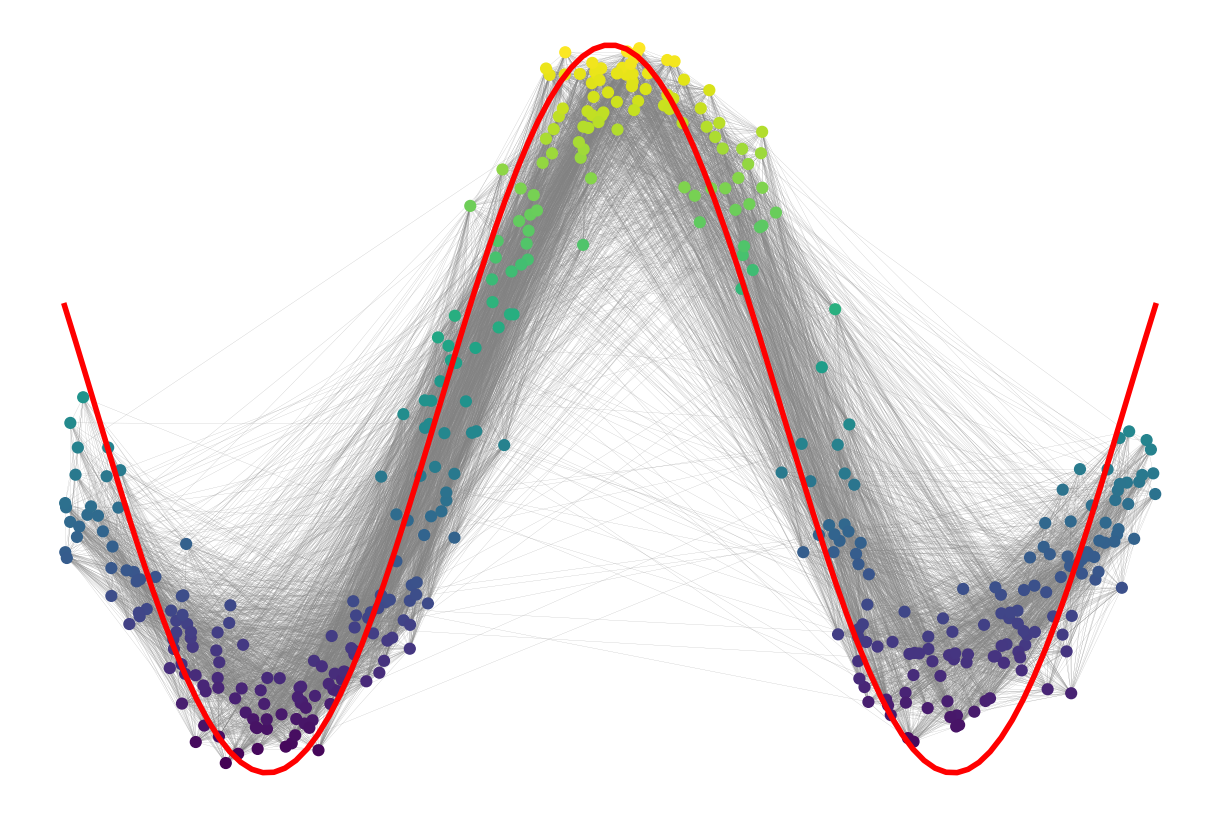}
    \includegraphics[width=0.32\textwidth]{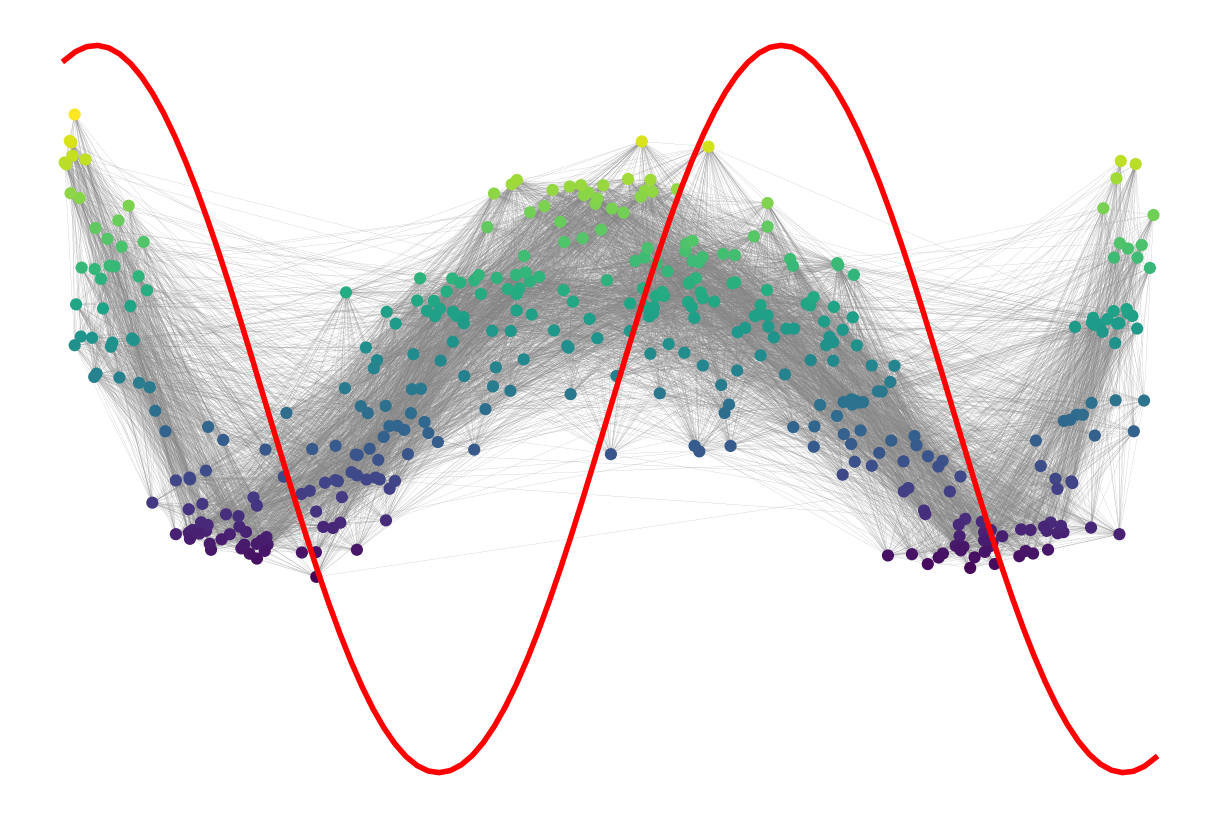}
    \includegraphics[width=0.32\textwidth]{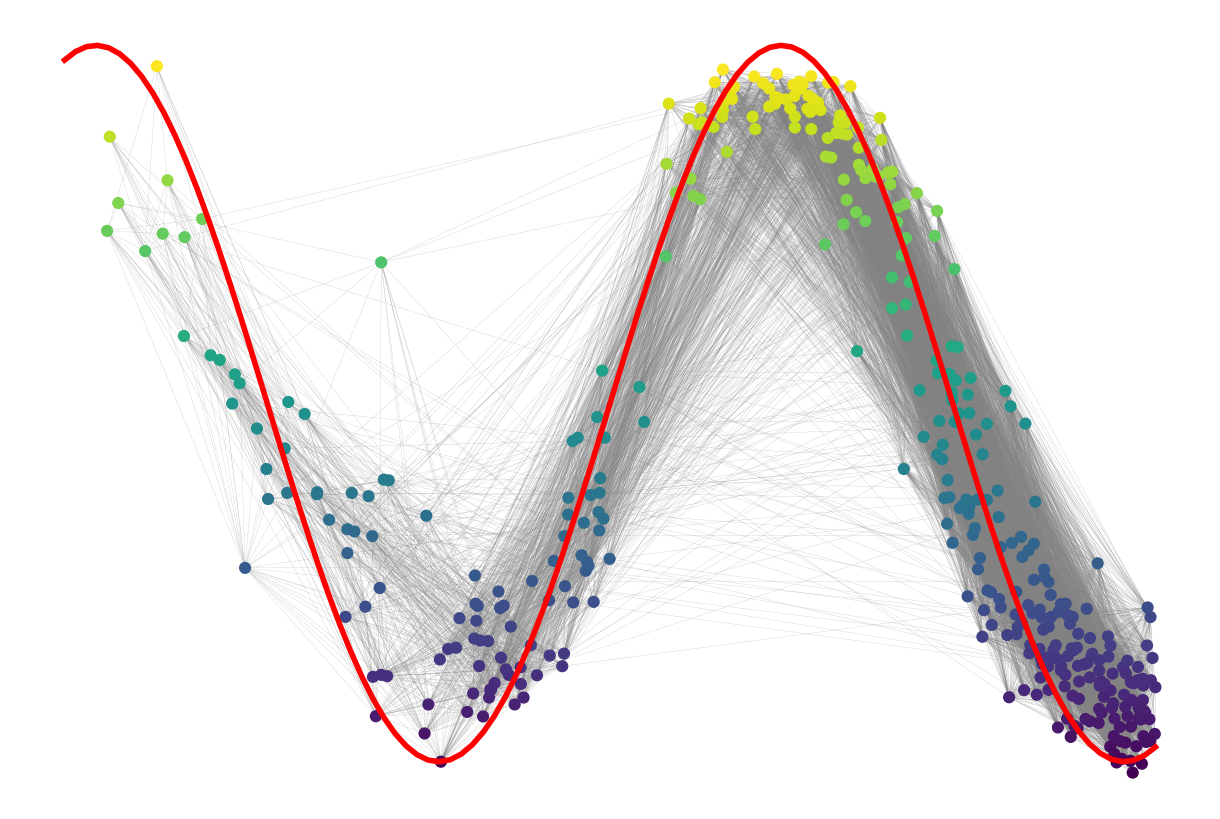}
    \caption{\small Illustration of Prop.~\ref{prop:rad-eq} on Gaussian kernel, but with random edges and two-hop input filtering. The x-axis is the latent variables in $\Xx = [-1,1]$. The y-axis is the output of a SGNN trained to approximate some function $f$ (red curve). On the left, both distribution $P$ and $f$ are symmetric, c-SGNNs are universal in that case. In the center, $P$ is symmetric but not $f$, and the training expectedly fails since the limit c-SGNN is symmetric. On the right, $P$ and $f$ are non-symmetric, and universality holds again. Details can be found in App.~\ref{app:numerics}.}
    %\vspace{-10pt}
    \label{fig:approx}
\end{figure}

\paragraph{Additive kernel.}

As in the invariant case, injectively additive kernels lead to universality.

\begin{proposition}\label{prop:decomp-eq}
    Assume $W$ is injectively additive, fix any $P$. Then, for one-hop input filtering, c-SGNNs are dense in $\Cc(\Xx, \RR)$.
\end{proposition}

\paragraph{Radial kernel.} Unlike the invariant case (Prop.~\ref{prop:rad-inv}) which was limited to symmetric distributions, we treat both symmetric and non-symmetric case: when $P$ is symmetric, then so is $\Psi_{W,P}$ (by permutation-equivariance), and we have universality \emph{among symmetric functions}. When $P$ is non-symmetric, we have universality among \emph{all} functions. See Fig.~\ref{fig:approx} for an illustration.

\begin{proposition}\label{prop:rad-eq}
    Consider $\Xx = [-1, 1]$, $W(x,y) = w(|x-y|)$ a radial kernel with an invertible, continuous $w$, and $P$ with a piecewise continuous density such that $\EE_P X=0$. Then, for one-hop input filtering: if $P$ is symmetric, c-SGNNs are dense in the space of \emph{symmetric} functions in $\Cc(\Xx,\RR)$, and if $P$ is not symmetric, c-SGNNs are dense in $\Cc(\Xx,\RR)$.
\end{proposition}

Finally, we look at radial kernels on the $d$-dimensional sphere $\Xx = \mathbb{S}^{d-1} = \{x \in \RR^d~|~\norm{x}=1\}$, an important example sometimes referred to as random geometric graphs \cite{Penrose-RGG}. In this case, the kernel only depends on the dot product $W(x,y) = w(x^\top y)$. Denoting by $d\tau$ the uniform measure on $\mathbb{S}^{d-1}$, it is known \cite{costas2014spherical,Dai-sphere} that functions in $L^2(d\tau)$ can be uniquely decomposed as $f(x) = \sum_{k\geq 0} f_k(x) = \sum_{k\geq 0} \sum_{j=1}^{N(d,k)} a_{k,j} Y_{k,j}(x)$ where $Y_{k,j}$ are \emph{spherical harmonics}, that is, homogeneous harmonic polynomials of degree~$k$ which form an orthonormal basis of $L^2(d\tau)$. We will say that such a function is \textbf{injectively decomposed} if the mapping $x \to [f_k(x)]_{k\geq 0}$ from $\mathbb{S}^{d-1}$ to $\ell_2(\RR)$ is injective.
Note that generically this is verified if~$f_k$ is non-zero for more than~$d-1$ distinct values of~$k > 0$, as this corresponds to solving an over-determined system of polynomial equations, but there may be degenerate situations where this is not enough. The proof of the following proposition is based on the well-known Legendre/Gegenbauer polynomial decomposition of spherical harmonics \cite{Dai-sphere} (see App.~\ref{app:rad-eq}).

\begin{proposition}\label{prop:dotprod-eq}
    Assume that $\Xx = \mathbb{S}^{d-1}$, that $W(x,y) = w(x^\top y)$ with continuous invertible $w:[-1,1] \to [0,1]$, and that $P = f d\tau$ has a density $f$ which is injectively decomposed. Then for one-hop input filtering c-SGNNs are dense in $\Cc(\Xx,\RR)$.
\end{proposition}

\section{Conclusion and outlooks}\label{sec:conclusion}

It is known that permutation-invariant GNNs fail to distinguish regular graphs of the same order, and permutation-equivariant GNNs return constant output on regular graphs. Similarly, their continuous counterparts suffer from the same flaw on random graph with constant or almost-constant degree function \cite{Magner-power}. However, we showed that the recently proposed SGNNs converge to continuous architectures which, like in the discrete world, are strictly more powerful than c-GNNs. Moreover, we proved that both permutation-invariant and permutation-equivariant c-SGNNs are universal on many random graph models of interest, including a large class of SBMs and random geometric graphs. %Hence, it is for instance possible to exhibit a simple, well-separated SBM for which all c-GNNs will fail to distinguish communities while there exists a c-SGNN that perfectly separates them (Fig.~\ref{fig:sbm}).

We believe that our work opens many possibilities for future investigations.
We examined very simple strategies for choosing the inputs $E_q(A)$ of the SGNN, but more complex, spectral-based choices could exhibit better convergence properties and approximation power.
We showed universality in specific random graph models of interest, but deriving a more generic criterion is still an open question. More directly, most of our examples illustrate the one-dimensional case $\Xx\subset \RR$, and a generalization to multidimensional latent spaces would be an important step forward.
Besides SGNNs, architectures that include high-order tensors \cite{Maron-powerful} (sometimes called FGNN \cite{Lelarge-power}) are known to be more powerful than the WL test. Conditions for their convergence on large graphs are still open, in particular since high-order tensors lead to high-order operators that may be difficult to manipulate.
Finally, we remark that directly estimating the latent variables $x_i$ is a classical task in statistics, for which conditions of success have been derived for various approaches, e.g. for Spectral Clustering \cite{Lei-SC}. Comparing them with (S)GNNs is an important path for future work.

\bibliographystyle{plain}
\bibliography{power}

\begin{thebibliography}{10}

\bibitem{Lelarge-power}
Wa{\"{i}}ss Azizian and Marc Lelarge.
\newblock {Characterizing the expressive power of invariant and equivariant
  graph neural networks}.
\newblock In {\em International Conference on Learning Representations (ICLR)},
  2021.

\bibitem{Babai-GI}
L{\'{a}}szl{\'{o}} Babai.
\newblock {Graph isomorphism in quasipolynomial time}.
\newblock {\em Proceedings of the Annual ACM Symposium on Theory of Computing},
  2016.

\bibitem{Bronstein-GDL}
Michael~M. Bronstein, Joan Bruna, Yann Lecun, Arthur Szlam, and Pierre
  Vandergheynst.
\newblock {Geometric Deep Learning: Going beyond Euclidean data}.
\newblock {\em IEEE Signal Processing Magazine}, 34(4):18--42, 2017.

\bibitem{Chen-substructure}
Zhengdao Chen, Lei Chen, Soledad Villar, and Joan Bruna.
\newblock {Can graph neural networks count substructures?}
\newblock In {\em Advances in Neural Information and Processing Systems
  (NeurIPS)}, 2020.

\bibitem{Bruna-clust}
Zhengdao Chen, Lisha Li, and Joan Bruna.
\newblock {Supervised community detection with line graph neural networks}.
\newblock In {\em International Conference on Learning Representations (ICLR)},
  2019.

\bibitem{Bruna-GI}
Zhengdao Chen, Soledad Villar, Lei Chen, and Joan Bruna.
\newblock {On the equivalence between graph isomorphism testing and function
  approximation with GNNs}.
\newblock In {\em Advances in Neural Information Processing System (NeurIPS)},
  pages 1--19, 2019.

\bibitem{Dai-sphere}
Feng Dai and Yuan Xu.
\newblock {\em {Approximation Theory and Harmonic Analysis on Spheres and
  Balls}}.
\newblock Springer, 2013.

\bibitem{DeBie-measure-nn}
Gwendoline de~Bie, Gabriel Peyr{\'{e}}, and Marco Cuturi.
\newblock {Stochastic deep networks}.
\newblock In {\em International Conference on Machine Learning}, 2018.

\bibitem{Defferrard-cheb}
Micha{\"{e}}l Defferrard, Xavier Bresson, and Pierre Vandergheynst.
\newblock {Convolutional Neural Networks on Graphs with Fast Localized Spectral
  Filtering}.
\newblock In {\em Advances in Neural Information and Processing Systems
  (NIPS)}, 2016.

\bibitem{Barabasi-subgraph}
Nima Dehmamy, Albert-L\`aszl\`o Barab\`asi, and Rose Yu.
\newblock {Understanding the representation power of graph neural networks in
  learning graph topology}.
\newblock In {\em Advances in Neural Information and Processing Systems
  (NeurIPS)}, 2019.

\bibitem{costas2014spherical}
Costas Efthimiou and Christopher Frye.
\newblock {\em Spherical harmonics in p dimensions}.
\newblock World Scientific, 2014.

\bibitem{Garg-generalization}
Vikas~K. Garg, Stefanie Jegelka, and Tommi Jaakkola.
\newblock {Generalization and representational limits of graph neural
  networks}.
\newblock In {\em International Conference on Machine Learning (ICML)}, pages
  1--22, 2020.

\bibitem{Geerts-kWL}
Floris Geerts.
\newblock {The expressive power of kth-order invariant graph networks}.
\newblock {\em arXiv:2007.12035}, 2020.

\bibitem{Gilmer-MP}
Justin Gilmer, Samuel~S Schoenholz, Patrick~F Riley, Oriol Vinyals, and
  George~E Dahl.
\newblock {Neural Message Passing for Quantum Chemistry}.
\newblock In {\em International Conference on Machine Learning (ICML)}, pages
  1--14, 2017.

\bibitem{Hamilton-book}
William~L. Hamilton.
\newblock {\em {Graph Representation Learning}}.
\newblock Morgan \& Claypool, 2020.

\bibitem{Holland-SBM}
Paul~W Holland.
\newblock {Stochastic blockmodels: First steps}.
\newblock {\em Social Networks}, 5(2):109--137, 1983.

\bibitem{Hornik-mlp}
Kurt Hornik, Maxwell Stinchcombe, and Halbert White.
\newblock {Multilayer Feedforward Networks are Universal Approximators}.
\newblock {\em Neural Networks}, 2:359--366, 1989.

\bibitem{OGB}
Weihua Hu, Matthias Fey, Marinka Zitnik, Yuxiao Dong, Hongyu Ren, Bowen Liu,
  Michele Catasta, and Jure Leskovec.
\newblock {Open Graph Benchmark: Datasets for Machine Learning on Graphs}.
\newblock {\em Neural Information Processing Systems (NeurIPS)}, 2020.

\bibitem{us}
Nicolas Keriven, Alberto Bietti, and Samuel Vaiter.
\newblock {Convergence and Stability of Graph Convolutional Networks on Large
  Random Graphs}.
\newblock In {\em Advances in Neural Information and Processing Systems
  (NeurIPS)}, pages 1--26, 2020.

\bibitem{Keriven-universal}
Nicolas Keriven and Gabriel Peyr{\'{e}}.
\newblock {Universal Invariant and Equivariant Graph Neural Networks}.
\newblock In {\em Advances in Neural Information Processing Systems (NeurIPS)},
  pages 1--19, 2019.

\bibitem{Kipf-SSL}
Thomas~N Kipf and Max Welling.
\newblock {Semi-Supervised Learning with Graph Convolutional Networks}.
\newblock In {\em International Conference on Learning Representations (ICLR)},
  2017.

\bibitem{Lei-SC}
Jing Lei and Alessandro Rinaldo.
\newblock {Consistency of spectral clustering in stochastic block models}.
\newblock {\em Annals of Statistics}, 43(1):215--237, 2015.

\bibitem{Levie-transferability}
Ron Levie, Michael~M. Bronstein, and Gitta Kutyniok.
\newblock {Transferability of spectral graph convolutional neural networks}.
\newblock {\em arXiv:1907.12972}, pages 1--41, 2019.

\bibitem{Loukas-GI}
Andreas Loukas.
\newblock {How hard is to distinguish graphs with graph neural networks?}
\newblock In {\em Advances in Neural Information and Processing Systems
  (NeurIPS)}, 2020.

\bibitem{Loukas-depth}
Andreas Loukas.
\newblock {What graph neural networks cannot learn: depth vs width}.
\newblock In {\em ICLR}, 2020.

\bibitem{Lovasz-graphon}
L{\'{a}}szl{\'{o}} Lov{\'{a}}sz.
\newblock {Large networks and graph limits}.
\newblock {\em Colloquium Publications}, 60, 2012.

\bibitem{Magner-power}
Abram Magner, Mayank Baranwal, and Alfred~O. Hero.
\newblock {The Power of Graph Convolutional Networks to Distinguish Random
  Graph Models}.
\newblock {\em arXiv:1910.12954}, pages 1--27, 2019.

\bibitem{Maron-powerful}
Haggai Maron, Heli Ben-Hamu, Hadar Serviansky, and Yaron Lipman.
\newblock {Provably Powerful Graph Networks}.
\newblock In {\em Advances in Neural Information Processing Systems (NeurIPS)},
  pages 1--12, 2019.

\bibitem{Maron-universal}
Haggai Maron, Ethan Fetaya, Nimrod Segol, and Yaron Lipman.
\newblock {On the Universality of Invariant Networks}.
\newblock In {\em International Conference on Machine Learning (ICML)}, 2019.

\bibitem{Penrose-RGG}
Mathew Penrose.
\newblock {\em {Random Geometric Graphs}}.
\newblock Oxford University Press, 2008.

\bibitem{Pinkus-MLP}
Allan Pinkus.
\newblock {Approximation theory of the MLP model in neural networks}.
\newblock {\em Acta Numerica}, 8(May):143--195, 1999.

\bibitem{Rahimi-RF}
Ali Rahimi and Benjamin Recht.
\newblock {Weighted sums of random kitchen sinks: Replacing minimization with
  randomization in learning}.
\newblock In {\em Advances in Neural Information Processing Systems (NIPS)},
  2009.

\bibitem{Rosasco-ope}
Lorenzo Rosasco, Mikhail Belkin, and Ernesto {De Vito}.
\newblock {On learning with integral operators}.
\newblock {\em Journal of Machine Learning Research}, 11:905--934, 2010.

\bibitem{Ruiz-transferability}
Luana Ruiz, Luiz~F.O. Chamon, and Alejandro Ribeiro.
\newblock {Graphon neural networks and the transferability of graph neural
  networks}.
\newblock In {\em Advances in Neural Information and Processing Systems
  (NeurIPS)}, 2020.

\bibitem{Vershynin2018}
Roman Vershynin.
\newblock {\em {High-dimensional probability: An introduction with applications
  in data science}}.
\newblock Cambridge University Press, 2018.

\bibitem{Vignac-smp}
Cl{\'{e}}ment Vignac, Andreas Loukas, and Pascal Frossard.
\newblock {Building powerful and equivariant graph neural networks with
  message-passing}.
\newblock In {\em Advances in Neural Information and Processing Systems
  (NeurIPS)}, 2020.

\bibitem{VonLuxburg-SC}
Ulrike {Von Luxburg}, Mikhail Belkin, and Olivier Bousquet.
\newblock {Consistency of spectral clustering}.
\newblock {\em Annals of Statistics}, 36(2):555--586, 2008.

\bibitem{WL}
B~Yu Weisfeiler and A~A Leman.
\newblock {The Reduction of a Graph to Canonical Form and the Algebra Which
  Appears Therein}.
\newblock {\em Nti}, 2(9):12--16, 1968.

\bibitem{Wu-review}
Zonghan Wu, Shirui Pan, Fengwen Chen, Guodong Long, Chengqi Zhang, and
  Philip~S. Yu.
\newblock {A Comprehensive Survey on Graph Neural Networks}.
\newblock {\em IEEE Transactions on Neural Networks and Learning Systems},
  pages 1--21, 2020.

\bibitem{Xu-WL}
Keyulu Xu, Weihua Hu, Jure Leskovec, and Stefanie Jegelka.
\newblock {How Powerful are Graph Neural Networks?}
\newblock In {\em ICLR}, pages 1--15, 2019.

\bibitem{Maron-generalization}
Gilad Yehudai, Ethan Fetaya, Eli Meirom, Gal Chechik, and Haggai Maron.
\newblock {On Size Generalization in Graph Neural Networks}.
\newblock {\em arXiv:2010.08853}, 2020.

\bibitem{Zaheer-deep-sets}
Manzil Zaheer, Satwik Kottur, Siamak Ravanbhakhsh, Barnab{\'{a}}s P{\'{o}}czos,
  Ruslan Salakhutdinov, and Alexander~J. Smola.
\newblock {Deep sets}.
\newblock {\em Advances in Neural Information Processing Systems (NeurIPS)},
  2017-December(ii):3392--3402, 2017.

\bibitem{Zweig-measure-nn}
Aaron Zweig and Joan Bruna.
\newblock {A functional perspective on learning symmetric functions with neural
  networks}.
\newblock {\em arXiv}, pages 1--19, 2020.

\end{thebibliography}

\appendix
\clearpage

\section{Convergence}\label{app:conv}

Let us start with some notations. Given a GNN $\Phi$, we define some bounds on its parameters that will be used in the multiplicative constants of the theorem. Recall that the filters are written $h_{ij}^{(\ell)}(\lambda) = \sum_{k=0}^\infty \beta_{ijk}^{(\ell)} \lambda^k$. We define $B_k^{(\ell)} = \pa{\beta_{ijk}^{(\ell)}}_{ji} \in \RR^{d_{\ell+1} \times d_\ell}$ the matrix containing the order-$k$ coefficients, and by $B_{k,\abs{\cdot}}^{(\ell)} = \pa{\abs{\beta_{ijk}^{(\ell)}}}_{ji}$ the same matrix with absolute value on all coefficients. Recall that $\norm{\cdot}$ is the operator norm for matrices. We define the following bounds:
\begin{align*}
&H^{(\ell)}_2 = \sum_k \norm{B_k^{(\ell)}}
&&H^{(\ell)}_{\partial, 2} = \sum_k \norm{B_k^{(\ell)}}k \\
&H^{(\ell)}_\infty = \norm{B_{0,\abs{\cdot}}^{(\ell)}} + \sum_{k\geq 1} \norm{B_{k}^{(\ell)}}
&&H^{(\ell)}_{\partial,\infty} = \sum_k \norm{B_k^{(\ell)}} k \sqrt{\log k}
\end{align*}
which all converge by our assumptions on the $\beta_k$. We may also denote $H_2$ by $H_{L^2(P)}$ for convenience but this quantity does not depend on $P$.
Note that, only for $H_\infty$, we use the spectral norm of the matrix $B_{0,\abs{\cdot}}$ with non-negative coefficients, which is suboptimal compared to using $B_{0}$. This is due to a part of our analysis where we do not operate in a Hilbert space but only in a Banach space $\Bb(\Xx)$, see Lemma \ref{lem:filter_partial}. We also define $\norm{b^{(\ell)}} = \sqrt{\sum_j (b_j^{(\ell)})^2}$ to measure the norm of the bias.

Given $X = \{x_1,\ldots, x_n\}$ and any dimension $d$, we denote by $S_X$ the sampling operator acting on functions $f:\Xx \to \RR^d$ defined by $S_X f \eqdef [f(x_1), \ldots, f(x_n)] \in \RR^{n \times d}$. We have $\norm{\frac{1}{\sqrt{n}}S_X f}_F \leq \norm{f}_\infty$.
Finally, given $X$ and $W$, we define $W(X) \eqdef (W(x_i,x_j))_{ij} \in \RR^{n\times n}$, and remark that $\frac{W(X)}{n} \circ S_X = S_X \circ T_{W,X}$.
In the deterministic edges case, the adjacency matrix $A$ is directly $W(X)$. In the random edges case, $A$ has expectation $W(X)$ (conditionally on $X$).

\subsection{Convergence of GNNs: proof of Theorem \ref{thm:conv-gnn}}\label{app:conv-gnn}

This proof is a variant of \cite{us}. We prove Theorem \ref{thm:conv-gnn} with the following error terms:
\begin{align}
    &R_1(n) = \frac{C_1\sqrt{d_\Xx} + C_2\sqrt{\log\pa{\frac{\sum_\ell d_\ell}{\rho}}}}{\sqrt{n}} 
    &&R_2(n) = \frac{C_\nu C_3}{\sqrt{\alpha_n n}} 
    &&R_3(n) = C_4 \sqrt{\frac{\log (1/\rho)}{n}}\, , \label{eq:conv-gnn-error}
\end{align}
where $R_2$ is present in the random edges case and $R_3$ in the permutation-invariant case, and the following constants:
\begin{align*}
    &C = L_g \prod_{\ell=0}^{M-1}H_2^{(\ell)} &&C_1 = \sum_{\ell=0}^{M-1} C^{(\ell)} H_{\partial, \infty}^{(\ell)} \\
    &C_2 = C_1(1+ D_\Xx L_W) &&C_3 = \sum_{\ell=0}^{M-1} C^{(\ell)} H_{\partial, 2}^{(\ell)} &&C_4 = C^{(M)}
\end{align*}
where the MLP $g$ is $L_g$-Lipschitz and $D_\Xx = \sup_{x,x'\in \Xx}m_\Xx(x,x')$ is the diameter of $\Xx$, and
\begin{align*}
    C^{(\ell)} &= L_g\pa{\prod_{s=\ell+1}^{M-1}H_2^{(s)}}\pa{C_f \prod_{s=0}^{\ell-1} H_\infty^{(s)} + \sum_{s=0}^{\ell-1} \norm{b^{(s)}}\prod_{p=s+1}^{\ell-1}H_\infty^{(p)}}
\end{align*}
with the conventions that an empty product is $1$ and an empty sum is $0$.

We begin the proof by the equivariant case, the invariant case will simply use an additional concentration inequality. We have
\begin{align*}
\text{MSE}_X\pa{\Phi_A(Z^{(0)}), \Phi_{W,P}(f)} &= \frac{1}{\sqrt{n}}\norm{\Phi_A(Z^{(0)}) - S_X \Phi_{W,P}(f^{(0)})}_F \\
&\leq L_g \frac{1}{\sqrt{n}} \norm{Z^{(M)} - S_X f^{(M)}}_F .
\end{align*}
We therefore seek to bound that last term.
Define the notation
\begin{equation*}
    \Delta^{(\ell)} = \frac{1}{\sqrt{n}}\sqrt{\sum_j \norm{ \sum_i \pa{h^{(\ell)}_{ij}\pa{\frac{A}{n}} S_X f^{(\ell)}_i - S_X h^{(\ell)}_{ij}(T_{W,P}) f^{(\ell)}_i}}^2} .
\end{equation*}
Then, using the Lipschitzness of $\rho$, Lemma \ref{lem:filter_partial} with $\norm{A/n}\leq 1$, and the fact that $S_X \circ \rho = \rho \circ S_X$, we have
\begin{align*}
    &\norm{Z^{(\ell+1)} -S_X f^{(\ell+1)}}_F \\
    &=\pa{\sum_j \norm{\rho\pa{\sum_{i=1}^{d_\ell} h^{(\ell)}_{ij}\pa{\frac{A}{n}} Z^{(\ell)}_{:,i} + b_{j}^{(\ell)}1_n } - S_X \rho\pa{\sum_{i=1}^{d_\ell} h^{(\ell)}_{ij}(T_{W,P}) f^{(\ell)}_i + b_{j}^{(\ell)} }}^2}^\frac12 \\
    &=\pa{\sum_j \norm{\rho\pa{\sum_{i=1}^{d_\ell} h^{(\ell)}_{ij}\pa{\frac{A}{n}} z^{(\ell)}_i + b_{j}^{(\ell)}1_n } - \rho\pa{S_X\pa{\sum_{i=1}^{d_\ell} h^{(\ell)}_{ij}(T_{W,P}) f^{(\ell)}_i + b_{j}^{(\ell)}}}}^2}^\frac12 \\
    &\leq\pa{\sum_j \norm{\sum_{i=1}^{d_\ell} h^{(\ell)}_{ij}\pa{\frac{A}{n}} Z^{(\ell)}_{:,i} - S_X h^{(\ell)}_{ij}(T_{W,P}) f^{(\ell)}_i}^2}^\frac12 \\
    &\leq\pa{\sum_j \norm{\sum_{i=1}^{d_\ell} h^{(\ell)}_{ij}\pa{\frac{A}{n}} \pa{Z^{(\ell)}_{:,i} - S_X f^{(\ell)}_i}}^2}^\frac12 \\
    &\quad + \pa{\sum_j \norm{\sum_{i=1}^{d_\ell} h^{(\ell)}_{ij}\pa{\frac{A}{n}} S_X f^{(\ell)}_i - S_X h^{(\ell)}_{ij}(T_{W,P}) f^{(\ell)}_i}^2}^\frac12 \\
    &\leq H^{(\ell)}_2 \norm{Z^{(\ell)} -S_X f^{(\ell)}}_F + \sqrt{n}\Delta^{(\ell)} .
\end{align*}
A recursion shows that, for all $Z^{(0)}$:
\begin{align}\label{eq:conv-inter1}
    \text{MSE}_X\pa{\Phi_A(Z^{(0)}), \Phi_{W,P}(f)} &\leq L_g\sum_{\ell=0}^{M-1} \Delta^{(\ell)}\prod_{s=\ell+1}^{M-1} H_2^{(s)}\notag \\
    &\quad + L_g\pa{\prod_{\ell=0}^{M-1} H_2^{(\ell)}} \textup{MSE}_X(Z^{(0)}, f^{(0)}) .
\end{align}

We now bound all $\Delta^{(\ell)}$ with high probability.
Recall that $\frac{W(X)}{n} \circ S_X = S_X \circ T_{W,X}$, and that we have $\norm{\frac{S_X}{\sqrt{n}} f} \leq \norm{f}_\infty$ and $\norm{T_{W,X}}_\infty \leq 1$. By Lemma \ref{lem:filter_partial} we have
\begin{align}
    \Delta^{(\ell)} &\leq \sqrt{\sum_j \norm{\sum_i \pa{h^{(\ell)}_{ij}\pa{\frac{A}{n}} - h^{(\ell)}_{ij}\pa{\frac{W(X)}{n}}}\frac{S_X}{\sqrt{n}} f^{(\ell)}_i}^2} \notag \\
    &\qquad + \sqrt{\sum_j\norm{\sum_i \frac{S_X}{\sqrt{n}}\pa{h^{(\ell)}_{ij}(T_{W,X}) - h^{(\ell)}_{ij}(T_{W,P})} f^{(\ell)}_i}^2} \notag \\
    &\leq H^{(\ell)}_{\partial,2}\norm{\frac{A-W(X)}{n}}\norm{f^{(\ell)}}_\infty \notag \\
    &\qquad + \sum_k \norm{B_k} \sqrt{\sum_i \pa{\sum_{p=0}^{k-1} \norm{(T_{W,X}-T_{W,P})T_{W,P}^{k-1-p} f_i^{(\ell)}}_\infty}^2}. \label{eq:conv-inter2}
\end{align}

The first term in \eqref{eq:conv-inter2} is $0$ in the deterministic edges case. In the random edges case, it is handled with a recent concentration inequality for Bernoulli matrices \cite{Lei-SC}, recalled in Theorem \ref{thm:lei} in App. \ref{app:add}. Since $\alpha_n \gtrsim \frac{\log n}{n}$, for any $\nu$, there is a constant $C_\nu$ such that, with probability $1-n^{-\nu}$ on the random edges (conditionally on $X$), $\norm{\frac{A-W(X)}{n}} \leq \frac{C_\nu}{\sqrt{\alpha_n n}}$. By the law of total probability, it is valid with joint probability $1- n^{-\nu}$ on both $X$ and the random edges.

We now bound the second term in \eqref{eq:conv-inter2}. Define $\rho_k = \frac{C \rho}{(k+1)^2 \sum_\ell d_\ell}$ with $C$ such that $\sum_{k \ell} d_\ell \rho_k = \rho/4$ (even when the filters are not of finite order). Using an application of Dudley's inequality detailed in Lemma \ref{lem:chaining}, applied with $U(x,y) = W(x,y)f(y)$ which is bounded by $\norm{f}_\infty$ and has Lipschitz constant $L_W \norm{f}_\infty$ in the first variable, and a union bound, we obtain with probability $1-\rho/4$ that: for all $i,\ell, k$, we have
\begin{align*}
    \norm{(T_{W,X} - T_{W,P}) T^k_{W,P} f^{(\ell)}_i}_\infty &\lesssim \frac{1}{\sqrt{n}}\norm{f_i^{\ell}}_\infty \pa{\sqrt{d_\Xx} + (1 + D_\Xx L_W)\sqrt{\log \rho_k^{-1}}} .
\end{align*}
Coming back to the second term of \eqref{eq:conv-inter2}, with probability $1-\rho/4$:
\begin{align*}
    \sum_k &\norm{B_k} \sqrt{\sum_i \pa{\sum_{p=0}^{k-1}\pa{\norm{(T_{W,X}-T_{W,P})T_{W,P}^{k-1-p} f_i^{(\ell)}}_\infty}^2}} \\
    &\lesssim \frac{\pa{\sqrt{d_\Xx} + (1 + D_\Xx L_W)\sqrt{\log \frac{\sum_\ell d_\ell}{\rho}}}}{\sqrt{n}} \sum_k \norm{B_k} k\sqrt{\log k} \sqrt{\sum_i \norm{f_i^{(\ell)}}_\infty^2} \\
    &\leq \frac{\pa{\sqrt{d_\Xx} + (1 + D_\Xx L_W)\sqrt{\log \frac{\sum_\ell d_\ell}{\rho}}}}{\sqrt{n}} H^{(\ell)}_{\partial, \infty} \norm{f^{(\ell)}}_\infty .
\end{align*}

At the end of the day we obtain that with probability $1-\rho$, for all $\ell$:
\begin{equation*}
\Delta^{(\ell)} \propto \norm{f^{(\ell)}}_\infty \pa{\frac{H^{(\ell)}_{\partial,2}}{\sqrt{\alpha_n n}} +\frac{ H_{\partial, \infty}^{(\ell)}\pa{\sqrt{d} + (1 + D_\Xx L_W)\sqrt{\log \frac{\sum_\ell d_\ell}{\rho}}}}{\sqrt{n}}} .
\end{equation*}
We then use Lemma \ref{lem:bound_cgnn} to bound $\norm{f^{(\ell)}}_\infty$ and conclude.

For the invariant case, we have
\begin{align*}
    \norm{\bar\Phi_A(Z^{(0)}) - \bar \Phi_{W,P}(f^{(0)})}&\leq \textup{MSE}_X(\Phi_A(Z^{(0)}), \Phi_{W,P}(f^{(0)})) \\
    &\qquad + L_g \norm{\frac{1}{n}\sum_{i=1}^n f^{(M)}(x_i) - \int f^{(M)}(x) dP(x)}
\end{align*}
We use a vector Hoeffding's inequality \cite[Lemma 4]{Rahimi-RF} and a bound on $\norm{f^{(M)}}_\infty$ (Lemma \ref{lem:bound_cgnn}) to conclude.

\subsection{Convergence of SGNNs}\label{app:conv-sgnn}

We prove Theorem \ref{thm:conv-sgnn} with the same form of error terms \eqref{eq:conv-gnn-error} where $R_i$ is replaced by $R'_i$ with modified multiplicative constants $C'_i$. Here we will have:
\begin{align*}
    C' &= D L_g \prod_{\ell=0}^{M-1}H_2^{(\ell)} \\
    C'_1 &= D L_\Phi + \sum_{\ell=0}^{M-1} H_{\partial, \infty}^{\prime (\ell)} C^{\prime (\ell)} + \sum_{\ell=0}^{M-1} H_{\partial, \infty}^{(\ell)} C^{(\ell)} \\
    C'_2 &= (1+D_\Xx L_W) \sum_{\ell=0}^{M-1} H_{\partial, \infty}^{\prime (\ell)} C^{\prime (\ell)} + D L_\Phi D_\Xx + D C_\Phi + \sum_{\ell=0}^{M-1} H_{\partial, \infty}^{(\ell)} (C^{(\ell)} + D_\Xx L^{(\ell)}) \\
    C'_3 &= \sum_{\ell=0}^{M-1} H_{\partial, 2}^{\prime (\ell)} C^{\prime (\ell)} + H_{\partial, 2}^{(\ell)} C^{(\ell)}, \qquad C'_4 = C^{\prime (M)}
\end{align*}
where $H^{\prime(\ell)}_\star$ is like $H^{(\ell)}_\star$ but for the weights in $\Phi'$, the final-layer MLP of $\Phi'$ is denoted by $g'$ with a Lipschitz constant $L_{g'}$, and:
\begin{align*}
    D &= L_{g'}\prod_{\ell=0}^{M-1} H_2^{\prime (\ell)} \\
    C^{\prime (\ell)} &= L_{g'}\pa{\prod_{s=\ell+1}^{M-1}H_2^{\prime (s)}}\pa{C_\Phi \prod_{s=0}^{\ell-1} H_\infty^{\prime (s)} + \sum_{s=0}^{\ell-1} \norm{b^{\prime(s)}}\prod_{p=s+1}^{\ell-1}H_\infty^{\prime (p)}} \\
    C^{(\ell)} &= D L_g \pa{\prod_{s=\ell+1}^{M-1} H_2^{(s)}} \tilde C^{(\ell)}_\infty \\
    L^{(\ell)} &= D L_g \pa{\prod_{s=\ell+1}^{M-1} H_2^{(s)}} \pa{L_W \tilde C^{(\ell)}_\infty + \sqrt{d_\ell} L_\eta \prod_{s=0}^{\ell-1} H_\infty^{(s)}} \\
    C_\Phi &= \norm{g(0)} + L_g \tilde C^{(M)}_\infty \\
    L_\Phi &= L_g \Bigg( L_\eta \prod_{\ell=0}^{M-1}\norm{B_0^{(\ell)}} + L_W \sum_{\ell=0}^{M-1} \pa{\prod_{s=\ell+1}^{M-1}\norm{B_0^{(s)}}}\tilde C^{(\ell)}_2\Bigg)
\end{align*}
with
\begin{align*}
    \tilde C^{(\ell)}_\star &= C_\eta \prod_{s=0}^{\ell-1} H^{(s)}_\star + \sum\limits_{s=0}^{\ell-1} \norm{b^{(s)}}\prod\limits_{p=s+1}^{\ell-1} H^{(p)}_\star \quad\text{for $\star\in \{2, \infty\}$}\, .
\end{align*}

We start by applying Theorem \ref{thm:conv-gnn} on the outer GNN $\Phi'$. Since the result is uniformly valid over all input of the GNN $Z^{(0)}$ with probability $1-\rho$:
\begin{equation}\label{eq:conv-sgnn-decomp1}
    \textup{MSE}_X(\Psi_A, \Psi_{W,P}) \leq D \textup{MSE}_X\pa{\frac{1}{n}\sum_q \Phi_A(E_q(A)), \int \Phi_{W,P}(\eta(\cdot,x))dP(x)} + R'(n)
\end{equation}
where, from Theorem \ref{thm:conv-gnn}, $D = L_{g'}\prod_{\ell=0}^{M-1} H_2^{\prime (\ell)}$ is $C_1$ but for the weights in $\Phi'$, and $R'(n)$ is the error term formed by summing various $R'_i(n)$, taking into account that by Lemma \ref{lem:bound_cgnn} the function inputed in $\Phi'$ is bounded by $C_\Phi$.

We must therefore bound the first term in \eqref{eq:conv-sgnn-decomp1}. We write
\begin{align}
    \textup{MSE}_X&\pa{\frac{1}{n}\sum_q \Phi_A(E_q(A)), \int \Phi_{W,P}(\eta(\cdot,x))dP(x)} \notag \\
    &\leq \textup{MSE}_X\pa{\frac{1}{n}\sum_q \Phi_A(E_q(A)), \frac{1}{n}\sum_q \Phi_{W,P}(\eta(\cdot, x_q))} \notag \\
    &\quad + \textup{MSE}_X\pa{\frac{1}{n}\sum_q \Phi_{W,P}(\eta(\cdot, x_q)), \int \Phi_{W,P}(\eta(\cdot,x))dP(x)} \notag \\
    &\leq \sup_q \textup{MSE}_X\pa{\Phi_A(E_q(A)), \Phi_{W,P}(\eta(\cdot, x_q))} \notag \\
    &\quad + \norm{\frac{1}{n}\sum_q \Phi_{W,P}(\eta(\cdot, x_q)) - \int \Phi_{W,P}(\eta(\cdot,x))dP(x)}_\infty . \label{eq:conv-sgnn-decomp2}
\end{align}

Let us start with the second term. By Lemma \ref{lem:bound_cgnn} and the Lipschitzness of $g$, $U(x,y)\eqdef \Phi_{W,P}(\eta(\cdot, y))(x)$ is $C_\Phi$-bounded and $L_\Phi$-Lipschitz with respect to $x$.

Hence, applying Lemma \ref{lem:chaining}: with probability $1-\rho$,
\begin{equation}
    \norm{\frac{1}{n}\sum_q \Phi_{W,P}(\eta(\cdot, x_q)) - \int \Phi_{W,P}(\eta(\cdot,x))dP(x)}_\infty \lesssim \frac{L_\Phi \sqrt{d_\Xx} + (L_\Phi D_\Xx + C_\Phi)\sqrt{\log(1/\rho)}}{\sqrt{n}} .
\end{equation}

For the first term in \eqref{eq:conv-sgnn-decomp2}, we introduce some notations. We denote by $f^{(\ell)}_i:\Xx \times \Xx \to \RR$ the bivariate function propagated at each layer of the inner part of the c-SGNN, as:
\begin{align}
    &f^{(0)}_0 =\eta
    &&f^{(\ell+1)}_j = \rho \pa{ \sum_i h_{ij}^{(\ell)}(T_{W,P})f_i^{(\ell)} + b^{(\ell)}_j}
\end{align}
where $T_{W,P}$ is here to be understood as an operator on $\Cc(\Xx \times \Xx)$ defined by $T_{W,P}[f](x,y) = \int W(x,z) f(z,y)dP(z)$. With these notations, $\Phi_{W,P}(\eta(\cdot, x_q)) = g(f^{(M)}(\cdot, x_q))$. Note that we still have $\norm{T_{W,P}}_\infty \leq 1$ for this version.
We perform the computation as in the proof of Theorem \ref{thm:conv-gnn} in \eqref{eq:conv-inter1} to obtain:
\begin{align}\label{eq:conv-sgnn-inter1}
    &\sup_q \text{MSE}_X\pa{\Phi_A(E_q(A)), \Phi_{W,P}(\eta(\cdot,x_q))} \notag \\
    &\qquad \leq L_g\sum_{\ell=0}^{M-1} \sup_q \Delta_q^{(\ell)}\prod_{s=\ell+1}^{M-1} H_2^{(s)} + L_g\pa{\prod_{\ell=0}^{M-1} H_2^{(\ell)}} \sup_q \textup{MSE}_X(E_q(A), \eta(\cdot,x_q))
\end{align}
with
\begin{equation}
    \Delta_q^{(\ell)} = \frac{1}{\sqrt{n}}\sqrt{\sum_j \norm{ \sum_i \pa{h^{(\ell)}_{ij}\pa{\frac{A}{n}} S_X f^{(\ell)}_i(\cdot,x_q) - S_X h^{(\ell)}_{ij}(T_{W,P}) [f^{(\ell)}_i(\cdot,x_q)]}}^2} .
\end{equation}
Then, again we decompose
\begin{align}
    \sup_q \Delta^{(\ell)}_q &\leq H^{(\ell)}_{\partial,2}\norm{\frac{A-W(X)}{n}}\norm{f^{(\ell)}}_\infty \notag \\
    &\qquad + \sum_k \norm{B_k} \sqrt{\sum_i \pa{\sum_{p=0}^{k-1} \norm{(T_{W,X}-T_{W,P})T_{W,P}^{k-1-p} [f_i^{(\ell)}]}_\infty}^2} \label{eq:conv-sgnn-inter2}
\end{align}
where we recall here that $f^{(\ell)}_i$ is a bivariate function.

Again, the first term is $0$ in the deterministic edges case, and otherwise by Theorem \ref{thm:lei} we have $\norm{\frac{A-W(X)}{n}} \leq C_\nu/\sqrt{\alpha_n n}$ with probability $1- n^{-\nu}$, and by Lemma \ref{lem:bound_csgnn} we have
\begin{equation*}
    \norm{f^{(\ell)}}_\infty \leq C_\eta \prod_{s=0}^{\ell-1} H^{(s)}_\infty + \sum\limits_{s=0}^{\ell-1} \norm{b^{(s)}}\prod\limits_{p=s+1}^{\ell-1} H^{(p)}_\infty .
\end{equation*}

Fix $k, \ell, i$ for now. We will apply Lemma \ref{lem:chaining} with $U:(\Xx \times \Xx) \times \Xx \to \RR$ defined as $U((x,x'), y) = W(x,y) f(y,x')$ for $f(y,x') = T_{W,P}^k [ f^{(\ell)}_i(\cdot, x')](y)$. Since $\norm{T_{W,P}}_\infty \leq 1$ we have $\norm{f}_\infty \leq \norm{f^{(\ell)}_i}_\infty$. Then, 
\begin{align*}
    \norm{T_{W,P}^k [ f^{(\ell)}_i(\cdot, x')] - T_{W,P}^k [ f^{(\ell)}_i(\cdot, x'')]}_\infty &= \norm{T_{W,P}^k [ f^{(\ell)}_i(\cdot, x')- f^{(\ell)}_i(\cdot, x'')]}_\infty \\
    &\leq \norm{f^{(\ell)}_i(\cdot, x')- f^{(\ell)}_i(\cdot, x'')}_\infty \\
    &\leq L_\eta m_\Xx(x,x') \prod_{s=0}^{\ell-1} H_\infty^{(s)}
\end{align*}
by Lemma \ref{lem:bound_csgnn}. Hence $U$ is bounded by $\norm{f^{(\ell)}_i}_\infty$ and $L_i^{(\ell)}$-Lipschitz with respect to $(x,x')$, with
\begin{equation}
    L_i^{(\ell)} = L_W \norm{f^{(\ell)}_i}_\infty + L_\eta \prod_{s=0}^{\ell-1} H_\infty^{(s)} .
\end{equation}
Finally, note that $\Xx \times \Xx$ is compact with covering numbers proportional to $\varepsilon^{-2d}$. Hence by Lemma \ref{lem:chaining} and a union bound, again defining $\rho_k$ as in the proof of Theorem \ref{thm:conv-gnn} such that $\sum_{ik\ell} \rho_k = \rho$: with probability $1-\rho$, we have simultaneously for all $i,k,\ell$:
\begin{align*}
    &\sup_x \norm{(T_{W,X} - T_{W,P}) T^k_{W,P} f^{(\ell)}_i(\cdot,x)}_\infty \\
    &\qquad\qquad \lesssim \frac{1}{\sqrt{n}} \pa{\norm{f^{(\ell)}_i}_\infty\sqrt{d_\Xx} + (\norm{f^{(\ell)}_i}_\infty + D_\Xx L^{(\ell)}_i)\sqrt{\log \rho_k^{-1}}} .
\end{align*}
Hence, as in the previous proof:
\begin{align*}
    &\sum_k \norm{B_k} \sqrt{\sum_i \pa{\sum_{p=0}^{k-1}\pa{\norm{(T_{W,X}-T_{W,P})T_{W,P}^{k-1-p} f_i^{(\ell)}}_\infty}^2}} \\
    &\qquad \qquad \leq \frac{\pa{\norm{f^{(\ell)}}_\infty\sqrt{d_\Xx} + (\norm{f^{(\ell)}}_\infty + D_\Xx L^{(\ell)})\sqrt{\log \frac{\sum_\ell d_\ell}{\rho}}}H^{(\ell)}_{\partial, \infty}}{\sqrt{n}} 
\end{align*}
where $L^{(\ell)} = (\sum_i (L_i^{(\ell)})^2)^\frac12$.

\subsection{Proof of Prop. \ref{prop:two-hop-input}}\label{app:input}

The error can be written as
\begin{align*}
    \frac{1}{\sqrt{n}}\norm{A^2 e_q/n - S_X T_{W,P}(W(\cdot, x_q))}_2 &\leq \frac{1}{\sqrt{n}}\norm{A^2 e_q/n - W(X)^2 e_q / n}_2 \\
    &\quad + \norm{(T_{W,X}- T_{W,P})W(\cdot, x_q)}_\infty .
\end{align*}
Using chaining as in the previous section, we have 
\begin{align*}
    &\sup_x \norm{(T_{W,X} - T_{W,P}) W(\cdot,x)}_\infty \lesssim \frac{1}{\sqrt{n}} \pa{\sqrt{d_\Xx} + (1 + D_\Xx L_W)\sqrt{\log 1/\rho}}\, . 
\end{align*}

The first term is $0$ in the deterministic edges case, and otherwise:
\begin{align*}
    \frac{1}{\sqrt{n}}\norm{A^2 e_q/n - W(X)^2 e_q / n}_2 &= \left(\frac{1}{n}\sum_i \left(\frac{1}{n} \sum_j a_{ij} a_{jq} - \frac{1}{n} \sum_j W(x_i, x_j) W(x_j, x_q)\right)^2\right)^\frac12 \\
    &\leq \left(\frac{1}{n}\sum_{i\neq q} \left(\frac{1}{n} \sum_j a_{ij} a_{jq} - W(x_i, x_j) W(x_j, x_q)\right)^2\right)^\frac12\\
    &\quad + \frac{1}{\sqrt{n}} \left(\frac{1}{n} \sum_j a_{jq}^2 - \frac{1}{n} \sum_j W(x_j, x_q)^2\right)^2 .
\end{align*}
Now, by Bernstein inequality with 
\[
    Var(a_{ij} a_{jq}) \leq \EE(a_{ij}^2 a_{jq}^2) = \alpha_n^{-2} \EE(a_{ij} a_{jq}) = \alpha_n^{-2}W(x_i, x_j) W(x_j, x_q)
\]
for $i\neq q$ and a union bound, with proba $1-\delta$, we have:
\begin{align*}
    &\abs{\frac{1}{n} \sum_j a_{ij} a_{jq} - \frac{1}{n} \sum_j W(x_i, x_j) W(x_j, x_q)} \lesssim \frac{\sqrt{\log(n/\delta)}}{\alpha_n \sqrt{n}} \text{ for all $q$ and $i\neq q$.}
\end{align*}
Since$\abs{\frac{1}{n} \sum_j a_{jq}^2 - \frac{1}{n} \sum_j W(x_j, x_q)^2} \leq 2$, we have
\begin{equation}
    \frac{1}{\sqrt{n}}\sup_q \norm{A^2 e_q/n - S_X T_W(W(\cdot, x_q))}_2 \lesssim \frac{\sqrt{\log(n/\delta)}}{\alpha_n \sqrt{n}} + \frac{\sqrt{d_\Xx} + (1 + D_\Xx L_W)\sqrt{\log 1/\rho}}{\sqrt{n}}
\end{equation}

\section{Approximation power: invariant case}\label{app:inv}

\subsection{Application of Stone-Weierstrass}

\begin{proof}[Proof of Prop. \ref{prop:sw-inv}]
    We do the proof for cSGNNws, it is exactly similar for cGNNs.

    This is a direct application of Lemma \ref{lem:SW}: for any two cSGNNs $\bar \Psi:\Ww \times \Pp \to \RR^d$, $\bar \Psi':\Ww \times \Pp \to \RR^{d'}$, their concatenation $[\bar \Psi, \bar \Psi']:\Ww \times \Pp \to \RR^{d+d'}$ is also a cSGNN (if they do not use the same input transforms $\eta, \eta'$, one can concatenate $\eta'' = [\eta,\eta']$), and for any MLP $g$, $g \circ \bar \Psi$ is also a cSGNN.

    One must just check that cSGNNs are continuous with respect to $\norm{\cdot}_\infty + \normTV{\cdot}$ on $\Ww \times \Pp$:
    \begin{itemize}
        \item $(W,P) \mapsto \eta$ is continuous by assumption ;
        \item for any $f_{W,P}\in \Cc(\Xx, \RR^d)$ continuously indexed by $(W,P)$,
        \begin{align*}
            \norm{T_{W,P}[f_{W,P}] - T_{W',P'}[f_{W',P'}]}_\infty &\leq \norm{f_{W,P}}_\infty (\norm{W-W'}_\infty + \normTV{P - P'}) \\&\quad+ \norm{f_{W,P}-f_{W',P'}}_\infty
        \end{align*}
        and similarly for $\norm{\int f_{W,P} dP - \int f_{W',P'} dP'}$ ;
        \item the non-linearity $\rho$ is Lipschitz.
    \end{itemize}
\end{proof}

\subsection{cSGNNs are more powerful than cGNNs}\label{app:gnn-vs-sgnn-inv}

\begin{proof}[Proof of Theorem \ref{thm:gnn-vs-sgnn-inv}]
    By construction, cGNNs are included in cSGNNs, since one can take $\Phi=0$ as the input GNN before pooling in \eqref{eq:sgnn}.

    To prove strict inclusion, we will construct two models $(W,P), (W',P')$ such that, for any cGNN we have $\bar\Phi_{W,P}=\bar\Phi_{W',P'}$, but there exists a cSGNN such that $\bar\Psi_{W,P}=\bar\Psi_{W',P'}$. We do the proof in the random edges case with two-hop input filtering $\eta_{W,P} = T_{W,P}(W)$, since such cSGNNs can of course also be constructed in the deterministic edges case.

    Since $\Xx$ is not a singleton, one can can single out two arbitrary elements $x,x'$ and take $P$ as a sum of two Diracs over them, which is equivalent to considering that $\Xx = \{x,x'\}$ (since any invariant architecture involves a final integration by $P$, it is useless to consider $W$ outside of the support of $P$). This results in a two-community SBM, for which $P$ can be represented as a $2$-vector on the simplex and $W$ as a $2$-by-$2$ symmetric matrix. We then consider a family of SBMs indexed by $\gamma \in [0,1]$:
    \begin{equation*}
        P = \pa{\begin{matrix} 1/3 \\ 2/3 \end{matrix}}, \quad W_\gamma = \pa{\begin{matrix} \gamma & \frac{1-\gamma}{2} \\ \frac{1-\gamma}{2} & \frac{1+\gamma}{4} \end{matrix}} .
    \end{equation*}
    It is not hard to see that $T_{W_\gamma,P}[1] = 1/3 \cdot 1$ for any $\gamma$. Therefore, for any cGNN $\Phi$, the function propagated inside its layers is always constant, and does not depend on $\gamma$. That is, $\bar \Phi_{W_\gamma,P} = \bar \Phi_{W_0,P}$ for any $\gamma$. On the other hand, consider the following SGNN:
    \begin{equation*}
        \bar \Psi_{W,P} = \int_x \int_y f(T_W(W(\cdot,y))) dP(y)dP(x)
    \end{equation*}
    where $f$ is an MLP. By the universality theorem, $f$ can approximate $x \to x^2$, for which we obtain:
    \begin{align*}
        \bar \Psi_{W_\gamma,P} &\approx 1/16*\gamma^4 - 1/12*\gamma^3 + 1/24*\gamma^2 - 1/108*\gamma + 17/1296 .
    \end{align*}
    This is not a constant function, so we can always find $\gamma,\gamma'$ such that $\bar \Psi_{W_\gamma,P} \neq \bar \Psi_{W_{\gamma'},P}$, which concludes the proof.
\end{proof}

\subsection{SBMs}\label{app:sbm-inv}

\begin{proof}[Proof of Prop. \ref{prop:sbm-inv}]
    We apply Prop. \ref{prop:sw-inv}. We fix $P$ as an incoherent vector in the $k$-simplex, and define $\Mm = \{(W,P): W \in S_k([0,1])\}$ which is indeed compact. It therefore suffices to show that cSGNNs separates points in $\Mm$.

    We proceed by contraposition: assume that $W,W'$ are such that $\bar \Psi_{W,P} = \bar \Psi_{W',P}$ for any cSGNN $\Psi$. We must show that necessarily $W=W'$.
    We look at cSGNNs of the form 
    \begin{align*}
        \Psi_{W,P} &= \int f_1\pa{\int f_0\pa{W(x, y)}dP(y)}dP(x) \\
        &= \sum_i P_i f_1(\sum_j P_j f_0(W_{ij})) = \sum_i P_i f_1(\sum_j P_j f_0(W'_{ij}))
    \end{align*}
    where $f_0,f_1$ are MLPs. By the universality theorem, they can approximate any continuous function. Pick any $f_0$. Then $f_1$ can be chosen as to take only values in $\{0,1\}$ on the discrete set $\{\sum_j P_j f_0(W_{ij}), \sum_j P_j f_0(W'_{ij})\}_i$ of size $2K$. Moreover, if there was an index $i_0$ such that $\sum_j P_j f_0(W_{i_0j})\neq \sum_j P_j f_0(W'_{i_0j})$, $f_1$ can be chosen to give different values on them. Then, defining $s_i = f_1(\sum_j P_j f_0(W_{ij})) -  f_1(\sum_j P_j f_0(W'_{ij})) \in \{-1,0,1\}$, we have both $s_{i_0} \neq 0$ and
    \begin{equation*}
        \sum_i P_i s_i = \sum_i P_i \left(f_1(\sum_j P_j f_0(W_{ij})) -  f_1(\sum_j P_j f_0(W'_{ij}))\right) = 0
    \end{equation*}
    which contradicts the incoherence of $P$. So, for all $f_0$ and $i$, we have $\sum_j P_j f_0(W_{ij}) = \sum_j P_j f_0(W'_{ij})$. By the exact same reasoning on $f_0$, we obtain that for all $i,j$, $W_{ij} = W'_{ij}$, which concludes the proof.

    For the failure of c-GNNs, the proof is immediate using the example SBM in the proof of Theorem \ref{thm:gnn-vs-sgnn-inv}, since $P=[1/3,2/3]$ is indeed incoherent.
\end{proof}

\subsection{Decomposed kernel}

\begin{proof}[Proof of Prop. \ref{prop:decomp-inv}]
    Applying Prop. \ref{prop:sw-inv} with $\Mm = \{W\} \times \tilde \Pp$, it suffices to show that cSGNNs separate the distributions in $\tilde \Pp$.
    By contraposition, assume that $P,P' \in \tilde \Pp$ are such that $\bar\Psi_{W,P} = \bar\Psi_{W,P'}$ for any cSGNN, and we want to prove that necessarily $P=P'$.
    
    We look at cSGNN of the form $P \mapsto \int f_1(\int f_0(W(x,y))dP(y))dP(x)$, where $f_0,f_1$ are MLPs, that can approximate any continuous functions by the universality theorem.
    Since $u$ is continuous and injective, it is well-known that it has a continuous inverse on its image. Hence $f_0$ can be chosen to approximate $f_0 \approx u^{-1}$. By choosing $f_1$ to approximate $x \to x^k$, we obtain that:
    \begin{equation*}
        \int \pa{v(x) + \EE_P v}^k dP(x) = \int \pa{v(x) + \EE_{P'} v}^k dP'(x)\, .
    \end{equation*}
    Taking $k=1$ we obtain that $\EE_P v = \EE_{P'} v$, and by an easy recursion we have $\EE_P v^k = \EE_{P'} v^k$ for all $k$. Since $v$ is invertible and polynomial functions are universal approximators on compacts one can write $v^{-1}(x) = \sum_k a_k x^k$ and $x = \sum_k a_k v(x)^k$, such that $\EE_P X^k = \EE_{P'} X^k$. Again, by the universality of polynomial functions, $\EE_P f = \EE_{P'} f$ for any continuous function, which is well-known to be equivalent to $P=P'$ and concludes the proof.
\end{proof}

\subsection{Radial kernel}

\begin{proof}[Proof of Prop. \ref{prop:rad-inv}]
    We proceed as in the proof of Prop. \ref{prop:decomp-inv} above: assuming $P,P' \in \tilde \Pp$ are such that $\bar\Psi_{W,P} = \bar\Psi_{W,P'}$ for any cSGNN, we want to prove that necessarily $P=P'$. We look at cSGNN of the form $P \mapsto \int f_1(\int f_0(W(x,y))dP(y))dP(x)$, where $f_0,f_1$ are MLPs. Since $w$ is injective $f_0$ can approximate $(x \to x^2) \circ w^{-1}$. By choosing $f_1$ to approximate $x \to x^k$, and since $P,P'$ are centered we obtain
    \begin{equation*}
        \int \pa{x^2 + \EE_P X^2}^k dP(x) = \int \pa{x^2 + \EE_{P'} X^2}^k dP'(x)\, .
    \end{equation*}
    Taking $k=1$ we have $\EE_P X^2 = \EE_{P'} X^2$, and by an easy recursion $\EE_P X^{2k} = \EE_{P'} X^{2k}$ for all $k$. Since $P,P'$ have $0$ odd-order moments, $\EE_P X^{k} = \EE_{P'} X^{k}$ for all $k$, from which we can conclude $P=P'$ as in the previous proof.
\end{proof}

\section{Approximation power: equivariant case}\label{app:eq}

\subsection{Application of Stone-Weierstrass}

\begin{proof}[Proof of Prop. \ref{prop:sw-eq}]
    As the proof of Prop. \ref{prop:sw-inv}, this is a direct application of Lemma \ref{lem:SW}: the set of cSGNNs is closed by concatenation and composition with an MLP, $\Xx$ is compact, and any equivariant cSGNN in continuous since by assumption $W$, $\eta_{W,P}$ and $\rho$ are.
\end{proof}

\subsection{cSGNNs are more powerful than cGNNs}

\begin{proof}[Proof of Theorem.~\ref{thm:gnn-vs-sgnn-eq}]
    As in the proof of Theorem~\ref{thm:gnn-vs-sgnn-inv} in App.~\ref{app:gnn-vs-sgnn-inv}, non-strict inclusion is immediate. To prove strict inclusion, as in the proof of Theorem~\ref{thm:gnn-vs-sgnn-inv} we consider again the same $2$-community SBM but for $\gamma=1/2$:
    \begin{equation*}
        P = \pa{\begin{matrix} 1/3 \\ 2/3 \end{matrix}}, \quad W = \pa{\begin{matrix} 1/2 & 1/4 \\ 1/4 & 3/8 \end{matrix}} .
    \end{equation*}
    Again any c-GNN would return a constant function $\Phi_{W,P}(1) = \Phi_{W,P}(2)$, while if we consider the following c-SGNN for two-hop filtering:
    \begin{equation*}
        \Psi_{W,P} = \int f(T_W(W(\cdot,y))) dP(y)
    \end{equation*}
    with $f$ an MLP that approximates $x\to x^2$, we obtain $\Psi_{W,P}(1)\approx 1/8$ and $\Psi_{W,P}(2) \approx 11/96$, hence a non-constant function.
\end{proof}

\subsection{SBMs}

\begin{proof}[Proof of Prop.~\ref{prop:sbm-eq}]
    We treat the two-hop filtering case, since they are included in one-hop architectures. By Prop.~\ref{prop:sw-eq}, we must prove the separation of elements of $\Xx$, which here are discrete community labels $\Xx = \{1,\ldots, K\}$. Fix $P\in \Delta^{K-1}$ incoherent and $W \in S_K$ invertible. Assume that $k,k'$ are two communities such that $\Psi_{W,P}(k) = \Psi_{W,P}(k')$ for all $\Psi$ with two-hop input filtering. We want to show that necessarily $k=k'$. By assumption, we have:
    \begin{align*}
        \sum_i f\pa{\sum_j W_{kj}W_{ji}P_j} P_i = \sum_i f\pa{\sum_j W_{k'j}W_{ji}P_j} P_i
    \end{align*}
    for any MLP $f$. As in the proof of Prop.~\ref{prop:sbm-inv} in App.~\ref{app:sbm-inv}, $f$ can approximate a function that is $\{0,1\}$-valued on its inputs such that, if there is an index $i$ such that $\sum_j W_{kj}W_{ji}P_j \neq \sum_j W_{kj}W_{ji}P_j$, then the incoherency of $P$ is contradicted. Hence, for all $i$, $\sum_j W_{kj}W_{ji}P_j = \sum_j W_{kj}W_{ji}P_j$, or in other words:
    \begin{equation*}
        W \cdot (P \odot (W_{k,:} - W_{k',:})) = 0 .
    \end{equation*}
    Since $W$ is invertible and $P$ has only non-zero coordinates (by incoherency), we obtain $W_{k,:} = W_{k',:}$. Since $W$ is invertible it has necessarily distinct columns, so $k=k'$, which concludes the proof.

    For the failure of c-GNNs we use the example of the proof of Theorem.~\ref{thm:gnn-vs-sgnn-eq}, for which $P$ is incoherent, $W$ is invertible, but any c-GNN is constant.
\end{proof}

\subsection{Additive kernel}

\begin{proof}[Proof of Prop.~\ref{prop:decomp-eq}]
    Again we apply Prop.~\ref{prop:sw-eq}. Assume that $x,x'$ are such that $\Psi_{W,P}(x) = \Psi_{W,P}(x')$ for all one-hop c-SGNN. In particular,
    \begin{equation*}
        \int f\pa{u(v(x)+v(y))} dP(y) = \int f\pa{u(v(x')+v(y))} dP(y)
    \end{equation*}
    for all MLP $f$. By taking $f = u^{-1}$, we obtain $v(x) = v(x')$, which leads to $x=x'$ by assumption of injectivity and concludes the proof.
\end{proof}

\subsection{Radial kernel}\label{app:rad-eq}

\begin{proof}[Proof of Prop.~\ref{prop:rad-eq}]
    Again we apply Prop.~\ref{prop:sw-eq}. Let $x,x'$ such that $\Psi_{W,P}(x) = \Psi_{W,P}(x')$ for all c-SGNNs. We want to prove that: if $P$ is symmetric, then $x=x'$ or $x=-x'$ (i.e. we quotient $[-1,1]$ by symmetry), and if $P$ is not symmetric, then necessarily $x=x'$.
    
    By assumption $\int f(W(x,y))dP(y) = \int f(W(x',y))dP(y)$ for all MLP $f$.

    \paragraph{If $P$ is symmetric.} By choosing $f = (\cdot)^2 \circ w^{-1}$, we have
    \begin{align*}
    0 = \EE(x-X)^2 - \EE(x'-X)^2 = x^2 + 2x \EE X + \EE X^2 - (x')^2 - 2x' \EE X - \EE X^2 = x^2-(x')^2
    \end{align*}
    which is indeed $x'=x$ or $x'=-x$.

    \paragraph{If $P$ is not symmetric.} By the previous reasoning, we still have $x'=x$ or $x'=-x$, however, we must now show that the case $x'=-x$ is not possible. By contradiction, assume $x'=-x$ (and $x\neq 0$). Denote $M_k = \EE_P X^k$ the kth moment of $P$. By Lemma \ref{lem:legendre}, $P$ is symmetric iff $M_{2k+1}=0$ for all $k$. We are going to show that this is the case by recursion: that is true for $k=0$ by assumption, and if $M_{2\ell+1}=0$ for all $\ell \leq k-1$, by taking $f_0(t)=t^{2k+2}$:
    \begin{align*}
        0 = \EE (x-X)^{2(k+1)} - \EE (x+X)^{2(k+1)} &= \sum_{\ell=0}^{2(k+1)} \binom{2(k+1)}{\ell} x^\ell (-1)^\ell M_{2(k+1)-\ell} \\
        &\qquad- \left(\sum_{\ell=0}^{2(k+1)} \binom{2(k+1)}{\ell} x^\ell M_{2(k+1)-\ell}\right) \\
        &= \sum_{\ell=1}^{k+1} \binom{2(k+1)}{2\ell-1} x^{2\ell-1} M_{2(k+1-\ell)+1} \\
        &= 2(k+1) x M_{2k+1}
    \end{align*}
    and therefore $M_{2k+1}=0$, and $P$ is symmetric, which is a contradiction. Therefore, necessarily $x'=x$, which completes the proof.
\end{proof}

\begin{proof}[Proof of Prop.~\ref{prop:dotprod-eq}]
    Again we apply Prop.~\ref{prop:sw-eq}. Let $x, x' \in \mathbb S^{d-1}$ such that $\Psi_{W,P}(x) = \Psi_{W,P}(x')$ for all c-SGNNs. We want to prove that $x=x'$.
    In particular,
    \begin{equation}
    \label{eq:dotprod_eq_assumption}
    \int g(w(x^\top y)) dP(y) = \int g(w(x'^\top y)) dP(y),
    \end{equation}
    for all MLP $g$.

    Recall that we have assumed that that $P$ has a density $f$ decomposed as
    $f(x) = \sum_{k \geq 0} f_k(x) = \sum_{k \geq 0} \sum_{j=1}^{N(d,k)} a_{k,j} Y_{k,j}(x)$, where $Y_{k,j}$ are spherical harmonics.

    Let~$P_k$ denote the Legendre/Gegenbauer polynomial of degree~$k$, which satisfies the addition formula
    \begin{equation*}
    P_k(x^\top y) = \frac{1}{N(d,k)} \sum_{j=1}^{N(d,k)} Y_{k,j}(x) Y_{k,j}(y).
    \end{equation*}
    Then, taking~$g = N(d,k) P_k \circ w^{-1}$, note that we have
    \begin{align*}
    \int g(w(x^\top y)) dP(y) &= N(d,k)\int P_k(x^\top y) f(y) d \tau(y) \\
        &= \sum_j Y_{k,j}(x) \langle f, Y_{k,j} \rangle_{L^2(d\tau)} \\
        &= f_k(x).
    \end{align*}
    Thus, \eqref{eq:dotprod_eq_assumption} implies $f_k(x) = f_k(x')$ for all~$k$. By assumption of injectivity of $x \to [f_k(x)]_k$, necessarily $x=x'$, which concludes the proof.
\end{proof}

\section{Additional material}\label{app:add}

\begin{lemma}\label{lem:SW}
    Let $(\Xx,d)$ be a compact metric space, $\Ff \subset \cup_{d\geq 1}\Cc(\Xx,\RR^d)$ be a subspace of continuous multivariate functions on $\Xx$ that is closed by concatenation, that is, $f,f' \in \Ff \Rightarrow [f,f'] \in \Ff$. Define $\Ff_\rho = \left\{ g \circ f ~|~ f \in \Ff,~\text{$g:\RR^d\to\RR$ is a MLP with non-linearity $\rho$} \right\}$, where $\rho$ is not polynomial. If $\Ff$ separates points, that is, $\forall x \neq x', \exists f \in \Ff, f(x) \neq f(x')$, then $\Ff_\textup{MLP}$ is dense in $\Cc(\Xx,\RR)$ for the supremum norm.
\end{lemma}

\begin{proof}
    The proof uses the classical Stone-Weierstrass theorem: an algebra of continuous functions that separates points is dense in the space of continuous functions (for the supremum norm).

    The main point is to check that $\Ff_\rho$ is an algebra. It is closed by linear combination: for all $g,g'$ MLPs, there is a $g''$ such that $g\circ f + g'\circ f' = g'' \circ [f, f']$ and $\Ff$ is closed by concatenation. Closure by multiplication is not true in general, however, following \cite{Hornik-mlp}, this is true when $\rho = \cos$: since $\cos(a)\cos(b) = \frac12(\cos(a+b) - \cos(a-b))$, we have: for $g(x) = \sum_i a_i \cos (b_i^\top x + c_i)$ and similarly $g'$,
    \begin{align*}
        (g\circ f) \cdot (g' \circ f') &= \sum_{ij} a_i a'_j \cos\left( b_i^\top f(x) + c_i \right)\cos\left( (b'_j)^\top f'(x) + c'_j \right) \\
        &= \sum_{ij} a_i a'_j \frac12 \Big(\cos\left( [b_i,b_j]^\top [f,f'](x) + c_i + c'_j \right) \\&\qquad- \cos\left( [b_i,-b_j]^\top [f,f'](x) + c_i - c'_j \right) \Big) \\
        &= g'' \circ [f,f']
    \end{align*}
    for a certain MLP $g''$.

    Hence $\Ff_{\cos}$ is an algebra. Moreover, it separates points: for $x \neq x'$, by hypothesis there is a $f\in \Ff$ such that $f(x) \neq f(x')$, and by the universality theorem applied to MLPs, this is also true for some $g \circ f$.

    To conclude the proof, we note that, by the universality theorem of MLPs, $\cos$ itself can be approached by a MLP with any non-polynomial non-linearity $\rho$, so that $\Ff_\rho$ is dense in $\Ff_{\cos}$.
\end{proof}

\begin{lemma}\label{lem:legendre}
    A piecewise continuous function $p$ on $[-1, 1]$ is symmetric iff $\int t^{2k+1} p(t) dt=0$ for all $k$.
\end{lemma}

\begin{proof}
    Recall that the Legendre polynomials $L_k$ of degree $k$ are such that: a) they form an orthogonal basis of piecewise continuous functions on $[-1,1]$ for $L^2$, b) respect parity $L_k(-t) = (-1)^k L_k(t)$, c) involves only monomials of the same parity $t^{k-2p}$, $p=0,\ldots, \lfloor \frac{k}{2} \rfloor$.

    By considering the decomposition $p = \sum_k (\int L_k p) L_k$, it is immediate that $p$ is symmetric iff $\int L_{2k+1} p = 0$ for all $k$, which is the same as $\int t^{2k+1} p(t) dt=0$ for all $k$.
\end{proof}

\begin{lemma}[Chaining]\label{lem:chaining}
    Let $(\Xx, m_\Xx)$ be a compact metric space with diameter $D_\Xx$ and covering numbers $\Nn(\Xx, m_\Xx, \varepsilon) \propto \varepsilon^{-d_\Xx}$, and $\Yy$ a measurable space. Consider a bivariate measurable function $U:\Xx \times \Yy \to \RR$ that is uniformly $C_U$-bounded, and $L_U$-Lipschitz in the first variable. Let $y_1,\ldots, y_n$ be drawn \emph{i.i.d} from a distribution $P$ on $\Yy$. Then, with probability at least $1-\rho$,
    \begin{equation*}
        \norm{\frac{1}{n}\sum_i \eta(\cdot, y_i) - \int \eta(\cdot, y)dP(y)}_\infty \lesssim \frac{L_U\sqrt{d_\Xx} + (L_U D_\Xx+ C_U)\sqrt{\log(1/\rho)}}{\sqrt{n}} .
    \end{equation*}
\end{lemma}

\begin{proof}
    For any $x\in \Xx$, define
    \[
        Y_x = \frac{1}{n}\sum_i U(x, x_i) - \int U(x, y) dP(y) .
    \]
    Since $\abs{Y_x} \leq 2 C_U$, for any fixed $x_0 \in \Xx$, by Hoeffding's inequality we have: with probability at least $1-\rho$, 
    \[
        \abs{Y_{x_0}} \lesssim C_U\sqrt{\frac{\log(1/\rho)}{n}} .
    \]
    Now we have
    \[
        \norm{\frac{1}{n}\sum_i U(\cdot, x_i) - \int U(\cdot, x)dP(x)}_\infty = \sup_{x \in \Xx} \abs{Y_x} \leq \sup_{x, x' \in \Xx} \abs{Y_x - Y_{x'}} + \abs{Y_{x_0}} .
    \]
    The second term is bounded by the inequality above.
    For the first term, we are going to use Dudley's inequality ``tail bound'' version \citep[Thm 8.1.6]{Vershynin2018}. We first check the sub-gaussian increments of the process $Y_x$. The sub-gaussian norm $\norm{\cdot}_{\psi_2}$ is defined in \citep[Def. 2.5.6]{Vershynin2018}. For any $x,x'\in \Xx$, we have
    \begin{align*}
    \norm{Y_x - Y_{x'}}_{\psi_2} &\lesssim \norm{\frac{1}{n} \sum_i U(x,y_i) - U(x',y_i)}_{\psi_2} \\
    &\lesssim \frac{1}{n} \pa{\sum_{i=1}^n \norm{U(x, y_i) - U(x', y_i)}_{\psi_2}^2}^\frac12 \\
    &\lesssim \frac{1}{n} \pa{n\norm{U(x, \cdot) - U(x', \cdot)}_{\infty}^2}^\frac12 \\
    &\leq \frac{L_U}{\sqrt{n}} m_\Xx(x,x')
    \end{align*}
    where we have used, from \citep{Vershynin2018}, Lemma 2.6.8 for the first line, Prop. 2.6.1 for the second, Example 2.5.8 for the third, and the Lipschitz property of $U$ for the last.
    
    Now, we apply Dudley's inequality \citep[Thm 8.1.6]{Vershynin2018} to obtain that with probability $1-\rho$,
    \begin{align*}
    \sup_{x, x' \in \Xx} \abs{Y_x - Y_{x'}} &\lesssim \frac{L_U}{\sqrt{n}}\pa{\int_0^1 \sqrt{\log N(\Xx, d, \varepsilon)}d\varepsilon + D_\Xx\sqrt{\log(1/\rho)}} \\
    &\lesssim L_U\frac{\sqrt{d} + D_\Xx\sqrt{\log(1/\rho)}}{\sqrt{n}} .
    \end{align*}
    Combining with the decomposition above yields the desired result.
\end{proof}

\begin{lemma}[Variant of Lemma 6 in \cite{us}]\label{lem:filter_partial}
    Let $(E,\norm{\cdot}_E)$ be a Banach space and $(\Hh, \norm{\cdot}_\Hh)$ be a separable Hilbert space. Let $L,L'$ be two bounded operators on $E$, and $S:E \to \Hh$ be a linear operator such that $\norm{S}_{\Hh \to E}\leq 1$. For $1\leq i \leq d$ and $1\leq j \leq d'$, let $h_{ij}=\sum_k \beta_{ijk} \lambda^k$ be a collection of analytic filters, with $B_k = \pa{\beta_{ijk}}_{ji}\in \RR^{d'\times d}$ the matrix of order-$k$ coefficients, with operator norm $\norm{B_k}$. Let $x_1,\ldots, x_{d}\in E$ be a collection of points. Then:
    \begin{equation}\label{eq:filter1}
        \sqrt{\sum_j \norm{S \sum_i h_{ij}(L)x_i}_\Hh^2} \leq \pa{\sum_k \norm{B_k} \norm{L^k}} \sqrt{\sum_i \norm{x_i}_E^2}
    \end{equation}
    and
    \begin{equation}\label{eq:filter2}
        \sqrt{\sum_j \norm{S \sum_i (h_{ij}(L)- h_{ij}(L'))x_i}_\Hh^2} \leq \sum_k \norm{B_k} \sqrt{\sum_i \pa{\sum_{\ell=0}^{k-1}\norm{L^\ell}\norm{(L-L')(L')^{k-1-\ell} x_i}_E}^2} .
    \end{equation}
    Now, if $\norm{L x}_E \leq \norm{S x}_\Hh$ for some Hilbert space $\Hh$, then
    \begin{equation}\label{eq:filter3}
        \sqrt{\sum_j \norm{\sum_i h_{ij}(L)x_i}_E^2} \leq \pa{\norm{B_{0,\abs{\cdot}}} + \sum_{k\geq 1} \norm{B_k} \norm{L^{k-1}}} \sqrt{\sum_i \norm{x_i}_E^2} .
    \end{equation}
\end{lemma}
\begin{proof}
    The results \eqref{eq:filter1} and \eqref{eq:filter2} are directly from Lemma 6 in \cite{us}. The result \eqref{eq:filter3} is obtained by observing that
    \begin{align*}
        \sqrt{\sum_j \norm{\sum_i h_{ij}(L)x_i}_E^2} &= \sqrt{\sum_j \norm{\sum_{ik} \beta_{ijk} L^k x_i}_E^2} \leq \sum_k \sqrt{\sum_j \norm{\sum_{i} \beta_{ijk} L^k x_i}_E^2} \\
        &\leq \sqrt{\sum_j \norm{\sum_{i} \beta_{ij0} x_i}_E^2} + \sum_{k\geq 1} \sqrt{\sum_j \norm{S\sum_{i} \beta_{ijk} L^{k-1} x_i}_\Hh^2} .
    \end{align*}
    We apply \eqref{eq:filter1} on the second term and on the first:
    \begin{align*}
        \sqrt{\sum_j \norm{\sum_{i} \beta_{ij0} x_i}_E^2} \leq \sqrt{\sum_j \pa{\sum_{i} \abs{\beta_{ij0}} \norm{x_i}_E}^2} \leq \norm{B_{0,\abs{\cdot}}} \sqrt{\sum_i \norm{x_i}_E^2} .
    \end{align*}
\end{proof}

\begin{lemma}[Properties of c-GNNs]\label{lem:bound_cgnn}
    Apply a c-GNN to a random graph model. Denote by $f^{(\ell)}$ the function at each layer. Then we have
    \begin{equation}
        \norm{f^{(\ell)}}_* \leq \norm{f}_* \prod_{s=0}^{\ell-1} H^{(s)}_* + \sum\limits_{s=0}^{\ell-1} \norm{b^{(s)}}\prod\limits_{p=s+1}^{\ell-1} H^{(p)}_*
    \end{equation}
    where $*$ indicates $L^2(P)$ or $\infty$.

    Moreover, for $x,x'\in \Xx$, we have
    \begin{align*}
        &\norm{f^{(\ell)}(x) - f^{(\ell)}(x')} \leq \pa{\prod_{s=0}^{\ell-1}\norm{B_0^{(s)}}} \norm{f^{(0)}(x) - f^{(0)}(x')} \\
        &\qquad + L_W d_\Xx(x,x') \sum_{s=0}^{\ell-1} \pa{\prod_{p=s+1}^{\ell-1}\norm{B_0^{(p)}}}\pa{\norm{f^{(0)}}_{L^2(P)}\prod_{p=0}^{s}H_2^{(p)} + \sum_{p=0}^{s-1} \norm{b^{(p)}}\prod_{r=p+1}^s H_2^{(r)}} .
    \end{align*}
    \end{lemma}
\begin{proof}
    For $j\leq d_\ell$, using Lemma \ref{lem:filter_partial}, the Lipschitzness of $\rho$ and the easy fact that $\norm{T_{W,P} f}_\infty \leq \norm{f}_{L^2(P)}$. we write
    \begin{align*}
        \norm{f^{(\ell)}}_* &\leq \sqrt{\sum_j \norm{\sum_{i=1}^{d_{\ell-1}} h_{ij}^{(\ell-1)}(T_{W,P}) f_i^{(\ell-1)} + b_j^{(\ell-1)}}_*^2} \\
        &\leq \sqrt{\sum_j \norm{\sum_{i=1}^{d_{\ell-1}} h_{ij}^{(\ell-1)}(T_{W,P}) f_i^{(\ell-1)}}_*^2} + \norm{b^{(\ell-1)}} \\
        &\leq H^{(\ell-1)}_* \norm{f^{(\ell-1)}}_* + \norm{b^{(\ell-1)}} .
    \end{align*}
    An easy recursion gives the result.

    Now,
    \begin{align*}
        \norm{f^{(\ell)}(x) - f^{(\ell)}(x')} &\leq \sqrt{\sum_j \pa{\sum_{i=1}^{d_{\ell-1}} h_{ij}^{(\ell-1)}(T_{W,P}) f_i^{(\ell-1)}(x) - h_{ij}^{(\ell-1)}(T_{W,P}) f_i^{(\ell-1)}(x')}^2} \\
        &=\sum_k \norm{B_k^{(\ell-1)} \left[T_{W,P}^k f_i^{(\ell-1)}(x) - T_{W,P}^k f_i^{(\ell-1)}(x')\right]_{i=1}^{d_{\ell-1}}}
    \end{align*}
    and since $\abs{T_{W,P}f(x) - T_{W,P}f(x')} \leq L_W d_\Xx(x,x')\norm{f}_{L^2(P)}$ and $\norm{T_{W,P}}_{L^2(P)} \leq 1$ by Schwartz inequality,
    \begin{align*}
        \norm{f^{(\ell)}(x) - f^{(\ell)}(x')} &\leq \norm{B_0^{(\ell-1)}} \norm{f^{(\ell-1)}(x) - f^{(\ell-1)}(x')} + H_2^{(\ell-1)}\norm{f^{(\ell-1)}}_{L^2(P)} L_W d_\Xx(x,x') .
    \end{align*}
    Again we obtain the result by recursion.
\end{proof}

\begin{lemma}(Properties of c-SGNNs)\label{lem:bound_csgnn}
    Denote by $f^{(\ell)}$ the bivariate functions propagated in the inner part of a c-SGNN. We have
    \begin{equation*}
        \norm{f^{(\ell)}}_\infty \leq C_\eta \prod_{s=0}^{\ell-1} H^{(s)}_\infty + \sum\limits_{s=0}^{\ell-1} \norm{b^{(s)}}\prod\limits_{p=s+1}^{\ell-1} H^{(p)}_\infty .
    \end{equation*}
    Moreover,
    \begin{equation*}
        \norm{f^{(\ell)}(\cdot, x) - f^{(\ell)}(\cdot, x')}_\infty \leq L_\eta d_\Xx(x,x') \prod_{s=0}^{\ell-1} H_\infty^{(s)} .
    \end{equation*}
\end{lemma}

\begin{proof}
    The first inequality is proved exactly as Lemma \ref{lem:bound_cgnn}, noting that $\norm{T_{W,P}}_\infty$ even for bivariate functions and $\norm{\eta}_\infty \leq C_\eta$.

    Then we have
    \begin{align*}
        \norm{f^{(\ell)}(\cdot, x) - f^{(\ell)}(\cdot, x')}_\infty &\leq \sqrt{\sum_j \norm{\sum_{i=1}^{d_{\ell-1}} h_{ij}^{(\ell-1)}(T_{W,P}) \left[f_i^{(\ell-1)}(\cdot, x) - f_i^{(\ell-1)}(\cdot, x')\right]}_\infty^2} \\
        &\leq H_\infty^{(\ell-1)}\norm{f^{(\ell-1)}(\cdot, x) - f^{(\ell-1)}(\cdot, x')}_\infty .
    \end{align*}
\end{proof}

\begin{theorem}[\cite{Lei-SC}]\label{thm:lei}
    Let $A$ be a $n \times n$ symmetric matrix with independent Bernoulli entries $a_{ij} \sim \alpha_n p_{ij}$. Assume that $\alpha_n \gtrsim \frac{\log n}{n}$. Then, for all $\nu >0$, there is a constant $C_\nu$ such that, for all $n$, with probability at least $1-n^{-\nu}$:
    \begin{equation}\label{eq:lei}
        \frac{1}{n}\norm{\frac{A}{\alpha_n}-P} \leq \frac{C_\nu}{\sqrt{\alpha_n n}} .
    \end{equation}
\end{theorem}

\section{Details of numerical experiments}\label{app:numerics}

The code is available at \url{https://github.com/nkeriven/random-graph-gnn}.

\paragraph{Figure~\ref{fig:sbm}.} In this figure, we consider a $2$-communities SBM with incoherent $P$, invertible $W$, but constant degree function. We use dense random edges with $\alpha_n=1$. We train a permutation-equivariant GNN and a two-hop filtering SGNN on $5$ random graphs with $n=80$ nodes, output dimension $d_{out}=K$, with cross-entropy loss and the Adam optimizer. The displayed graph signal corresponds to the first dimension of the log-softmax of the ouput. The test graph has $n=300$ nodes. The graph filters have order $1$, such that we actually manipulate the message-passing version of GNNs. The GNN has $M=5$ hidden layers with internal dimension $d_\ell=250$ (except $d_0=1$ and $d_{out}=2$) and is trained for $2000$ epochs. Each of the GNNs constituting the SGNN has $M=2$ hidden layers with dimension $d_\ell=50$ and is trained for $1000$ epochs.

\paragraph{Figure~\ref{fig:conv}.} We compare one-hop and two-hop input filtering for a simple permutation-invariant SGNN, between the deterministic edges case and the random edges case. We know that the deterministic edges case converges to the c-SGNN in all settings, and we test if the random edges case converge to the deterministic one. We average over $50$ random graphs with Gaussian kernel and a range of $n$'s with $\alpha_n \sim n^{-1/3}$, such that Prop.~\ref{prop:two-hop-input} applies in the two-hop case. The dominating term in the theoretical rate is $\order{1/(\alpha_n \sqrt{n})}$ from Prop.~\ref{prop:two-hop-input}.

\paragraph{Figure~\ref{fig:approx}.} Here we consider $\Xx=[-1,1]$ with Gaussian kernel, and either a symmetric $P$ (uniform) or a non-symmetric but centered $P$ (here a well-adjusted affine by part function). We use random edges with $\alpha_n \sim n^{-1/3}$. We train a SGNN with two-hop input filtering to approximate either a symmetric function $x \to \cos(5x)$ or a non-symmetric one $x \to \sin(5x)$ with a simple square loss. We use $5$ training graphs of size $n=150$ and display a test graph with size $n=400$.

\end{document}